\documentclass[11pt]{article} 
\usepackage{microtype}
\usepackage{graphicx}
\usepackage{subfigure}
\usepackage{booktabs} % for professional tables
\usepackage{algorithmic}% http://ctan.org/pkg/algorithm
\usepackage[linesnumbered,lined,boxed,ruled,noend]{algorithm2e}
\usepackage{dsfont}
\usepackage{verbatim}
\usepackage{amsmath,bm,bbm}
\usepackage{amssymb}
\usepackage{amsfonts}
\usepackage{mathtools}
\usepackage{amsthm}
\usepackage{xcolor}
\usepackage{natbib}
\usepackage{pifont}% 
\usepackage{enumitem}
\usepackage[capitalize,noabbrev]{cleveref}
\usepackage{thm-restate}
\usepackage{comment}
\newcommand{\naive}{na\"{\i}ve}
\newcommand{\Naive}{Na\"{\i}ve}

\newcommand{\purple}[1]{{\leavevmode\color[RGB]{128,0,255}{#1}}}
\newcommand{\orange}[1]{{\leavevmode\color[RGB]{229,83,0}{#1}}}
\newcommand{\darkgray}[1]{{\leavevmode\color[RGB]{120,120,120}{#1}}}

\newcommand{\cyan}[1]{{\leavevmode\color{black}{#1}}}

\newcommand{\inbal}[1]{\purple{[INBAL: {#1}]}}

\newcommand{\vineet}[1]{\cyan{[VINEET: {#1}]}}

\newcommand{\ganesh}[1]{\orange{[GANESH: {#1}]}}
\newcommand\expect[2]{\mathbb{E}_{#1}{\left[ {#2} \right]}}

\DeclareMathOperator*{\argmax}{argmax}

\DeclareMathOperator*{\argmin}{argmin}
\DeclareMathOperator{\sign}{sign}

\newcommand{\one}[1]{\mathds{1}{\{{#1}\}}}

\theoremstyle{plain}
\newtheorem{theorem}{Theorem}[section]

\newtheorem{lemma}[theorem]{Lemma}

\newtheorem{observation}[theorem]{Observation}
\theoremstyle{definition}
\newtheorem{definition}[theorem]{Definition}

\theoremstyle{remark}

\newtheorem{example}{Example}

\newcommand{\R}{\mathbb{R}}
\newcommand{\N}{\mathbb{N}}
\newcommand{\X}{{}\mathcal{X}}

\newcommand{\Z}{\mathcal{Z}}

\newcommand{\hhat}{{\hat{h}}}
\newcommand{\phat}{{\hat{p}}}
\newcommand{\muhat}{{\hat{\mu}}}
\newcommand{\w}{\bm{w}}

\newcommand{\rep}{{\phi}}
\newcommand{\pmone}{\{\pm 1 \}}
\newcommand{\lift}{\mathrm{lift}}
\newcommand{\sbstk}{\Gamma}

\newcommand{\err}[2]{{\varepsilon^{{#1}}_{{#2}}}}
\newcommand{\errexp}{{\err{}{}}}

\usepackage{fullpage}
\usepackage[T1]{fontenc}
 \usepackage[parfill]{parskip}
\usepackage{authblk}

\title{Strategic Representation}

\author[1]{Vineet Nair\thanks{Vineet Nair is now at Google Research, India.}}
\author[2]{Ganesh Ghalme\thanks{This work was done when Ganesh was a post-doctoral fellow at Technion.}}
\author[1]{Inbal Talgam-Cohen}
\author[1]{Nir Rosenfeld}
\affil[1]{Technion Israel Institute of Technology, Haifa, Israel.}
\affil[2]{Indian Institute of Technology Hyderabad, India.}

\begin{document}
\maketitle
\begin{abstract}
%Humans have come to rely on machines for reducing excessive information to manageable representations. But this reliance can be abused---strategic machines might craft such representations to manipulate their users.  How can a user learn to make careful choices based on strategic representations? We formalize this problem and pursue algorithms for decision making that are robust to manipulation. We study a key interaction of a system and user: the system represents the characteristics of some item to the user, who must then decide whether to consume it not. We view this interaction through the ``strategic classification’’ lens (Hardt et al. 2016)---as a game between the user aiming to classify the item presented to her, and the system providing the item’s features, possibly in a selective way. Our problem is however fundamentally different, both in the role reversal of the user (who classifies) and system (who manipulates), and in that the representation has to be the truth, but might not show the whole truth. Technically this means the user is learning set functions, which presents distinct algorithmic and statistical challenges. For the classifying user, our results give a complete characterization of the trade-off between a user’s learning effort and her susceptibility to manipulation. For the manipulating system, we show that surprisingly some of its choices are aligned with the user’s interests. 

Humans have come to rely on machines for reducing excessive information to manageable representations. But this reliance can be abused---strategic machines might craft representations that manipulate their users. 
How can a user make good choices based on strategic representations?
We formalize this as a learning problem,
and pursue algorithms for decision-making that are robust to manipulation.
In our main setting of interest,
the system represents attributes of an item to the user,
who then decides whether or not to consume.
We model this interaction through the lens of strategic classification (Hardt et al. 2016),
\emph{reversed}: the user, who learns, plays first;
and the system, which responds, plays second.
The system must respond with representations that
reveal `nothing but the truth'
but need not reveal the entire truth.
Thus, the user faces the problem of learning set functions
under strategic subset selection,
which presents distinct algorithmic and statistical challenges.
Our main result is a learning algorithm that minimizes error despite strategic representations,
and our theoretical analysis %give a complete characterization of
sheds light on the trade-off between learning effort and susceptibility to manipulation. %For the manipulating system, we show that surprisingly its 

%---as a Stackelberg game between the user, who classifies the item as ``worthwhile'' or not, and between the system, which provides the item’s attributes in a possibly selective way. The main problem is to learn to classify given selective features. There are also 
% but with key differences: %from strategic classification: 
% The roles of the user %(who classifies the item in our setting) 
% and system %(who manipulates the attributes) 
% are reversed; %raising interesting ``balance of power'' questions; also, %while 
% also, the system's representation might not be the whole truth but it can include nothing but true attributes.
% The classifying user is thus learning a set function,
% which presents distinct algorithmic and statistical challenges.
% Our main result is a learning algorithm that minimizes the classification error despite strategic representations.
% %For the classifying user, 
% Our analysis %give a complete characterization of
% sheds light on the trade-off between learning effort and susceptibility to manipulation. %For the manipulating system, we show that surprisingly its preferences may be aligned with the user’s interests. 
\end{abstract}

\section{Introduction}

There is a growing recognition that learning systems deployed
in social settings are likely to induce strategic interactions between the system and its users.
One promising line of research in this domain is \emph{strategic classification} \cite{HardtMPW16},
which studies learning in settings where users can strategically modify their features in response to a learned classification rule.
The primary narrative in strategic classification is one of  self-interested users that act to `game' the system,
and of systems that should defend against this form of gaming behavior.
% Strategic classification tells a story of how self-interested users act to `game' the system;
% works often focus on how to learn classifiers that are strategically-robust.
% learning systems are susceptible to `gaming' by self-interested users that can modify their features to manipulate the system.
% Strategic classification tells the story of how learning systems are susceptible to `gaming' by self-interested users. 
In this paper, we initiate the study of \emph{reversed} scenarios,
in which it is the system that strategically games its users.
We argue that this is quite routine:
In settings where users make choices about items
(e.g., click, buy, watch),
decisions are often made on the basis of only partial information.
But the choice of \emph{what} information to present to users
often lies in the hands of the system, which can utilize its representation power to promote its own goals.
% Systems often hold the power to determine which information is protraied,
% i.e., how items are \emph{represented}---a power they can utilize
% to promote their own goals.
Here we formalize the problem
of \emph{strategic representation},
and study how learning can now aid \emph{users} in choosing well
despite strategic system behavior.
% choose wisely in this setting.

% in act upon the system's representation.

% The reason for this may be exogenous;
% for example, a recommended movie cannot be presented 
% to users in its entirety; rather, it is represented
% as a thumbnail image depicting one of its scenes.
% Crucially, the choice of which image to show

% especially in settings where users choose how to act based on information provided by the system. In these, users are presented with \emph{representations} of items,
% possibly hiding certain details that are not within the system's best interest to reveal.

As a concrete example, consider a 
user browsing for a hotel in an online booking platform.
Hotels on the platform are represented by
% hotel booking platform, in which a hotel is presented to a user 
% alongside
a small set of %representative
images;
as there are many images available for each hotel, 
the system must choose which subset of these to display. 
Clearly, the choice of representation can have a significant effect on users' decision-making (whether to click the hotel, and subsequently whether to book it), and in turn, on the system's profit. 
%The system may well attempt to capitalize on the control it has over what information is presented to the user in order to increase its profit, at the expense of the user who will reach sub-optimal decisions. 
The system may well attempt to capitalize on the control it has over what information is presented to the user---at the expense of the user, who may be  swayed to make sub-optimal decisions. 
%Given the deep dependence of users on systems for information, it is important that the way in which users 
%respond to representations take into account the systems' strategizing.
Given that users rely on system-crafted representations for making decisions, choosing well requires users to account for the strategic nature of the representations they face. % strategic nature of representations. 
Our goal in this paper is 
to study when and how learning can aid users in making decisions
that are strategically robust.
We model users as deciding based on a \emph{choice strategy} $h$
mapping represented items to binary choices,
%(e.g., click, buy)
which can be learned from the user's previous experiences
(e.g., hotels stayed in, and whether they were worthwhile).
Since in reality full descriptions of items
are often too large for humans to process effectively
(e.g., hundreds or thousands of attributes),
the role of the system is to provide users with compact
representations (e.g., a small subset of attributes). 
% The user learns a \emph{choice} strategy $h$,
We therefore model the system as responding to $h$
with a \emph{representation mapping} $\rep$,
which determines for each item $x$ its representation $z=\rep(x)$,
% with item \emph{representations} $z=\rep(x)$,
% determined independently for each new item $x$,
% (e.g., a small set of images). %The user's `strategy' is the way she routinely acts upon representations
%
% \nir{users *can* learn - we also study gap from when they don't (naive)}
% Users are threfore tasked with choosing on the basis of system-chosen representations.
on which users rely for making choices
(i.e., $h$ is a function of $z$). %, not $x$).
We focus on items $x$ that are discrete and composed of a subset of
ground elements;
accordingly, we consider representation mappings $\rep$ that are
lossy but \emph{truthful},
meaning they reveal a cardinality-constrained subset of the item's full set of attributes:
$z \subseteq x$ and $k_1\le |z| \le k_2$
for some exogenously-set $k_1,k_2$. % (e.g., size of a slate).
% (which we think of as determined by `nature').
% In other words, representations need not be `the whole truth',
% but must include `nothing but the truth.'

Given that the system has representation power---how should users choose $h$?
% The key element in our formulation is that
The key idea underlying our formulation is that
the system and its users have misaligned goals.
A user wishes to choose items that are `worthwhile' to her,
as determined by her \emph{valuation function} $v(x)$---in our case, a set function,
which can reflect complex tastes (e.g., account for complementarity among the attributes,
such as `balcony' and `ocean view' in the hotel example).
% a balcony and a view are \emph{complementary} attributes for a hotel).
% The valuation is a set function over the item attributes which assigns positive or negative values to the attributes and allows for synergies among them
% (e.g., a balcony and a view are \emph{complementary} attributes for a hotel). %consumption, where choices are made on the basis of representations. 
Meanwhile, the system aims to 
maximize user engagement by
choosing a feasible representation $z=\rep(x)$
% presenting to the user a subset of attributes
that will incite the user to choose the item.
% E.g., in the hotel example, the system could present the attribute of having a balcony wihtout mentioning whether or not there is a view.
%\nir{maybe use hotels are running example here (and elsewhere) to demonstrate the above?}
% \footnote{Engagement is often indirectly maximized by predicting relevance as a proxy.} %, which is how things are usually framed}
%is the rate of consumption (e.g., click or buy) under some constraints (hard or soft) on the chosen representation. 
%We refer to this problem as `strategic representation'.
%Assuming binary item values $v(x) \in \{0,1\}$ \ganesh{We use $v(x) \in \{-1,+1\}$ },
%the payoff to users is their expected utility from correctly choosing high-valued items, % (i.e., $v(x)=1$),
%namely $\prob{x \sim D}{h(\rep(x)) = v(x)}$ where $D$ is an unknown data distribution,
%and their goal is to maximize utility by appropriately choosing $h$.
%Meanwhile, the system's payoff is the rate of consumption (e.g., click or buy),
%defined as $\expect{x \sim D}{h(\rep(x))}$,
%and we assume the system best-responds via the optimal item-specific representation mapping,
%$\rep(x)=\argmax_{z \subseteq x, |z| = k} h(z)$.
Importantly, while values are based on the true $x$, choices rely on representations $z$---which the system controls.
This causes friction:
a representation may be optimal for the system,
but may not align with user interests. 

The challenge for users (and our goal in learning) is therefore to make good choices on the basis of strategic representations.
Note that this is not a simple missing value problem,
since \emph{which attributes will be missing depends on how users choose},
i.e., on their learned choice function $h$.\footnote{Consider the following subtlety:
since users must commit to a choice strategy $h$,
the choice of \emph{how} to handle missing features
(a part of the strategy) determines \emph{which} features will be missing.}
Nor is this a problem that can be addressed by mapping representations to
a categorical representation (with classes `0',`1', and `unknown');
to see this, note that given a subset representation $z$, it is impossible to know which of the attributes that do not appear in $z$ are withheld---and which are truly missing.

\paragraph{Subset representations.}
We focus on subset representations since they provide a natural means to ensure that representations reveal `nothing but the truth' (but not necessarily `the whole truth').
This is our primary modeling concern, which stems from realistic restrictions on the system (like consumer expectations, legal obligations, or business practices).
Subset representations cleanly capture what we view as a fundamental tension between users and an informationally-advantageous system---the ability to withhold information;
examples include retail items described by a handful of attributes;
videos presented by several key frames;
and movie recommendations described by a select set of reviews,
to name a few.

\paragraph{Overview of our results.} %and organization.}
%\inbal{thm 5.10 for estimation error, for approximation, 5.4 is the main thing for approximation error. A phase transition result. thm 4.6 for the algorithm - the main structural result. for the system thm 5.11}
We begin by showing that users who choose {\naive}ly can easily be manipulated by a strategic system (Sec.~\ref{sec:Naive}).
% indeed, our results suggests that users may pay a hefty price if they neglect to guard against strategic system behavior.
We then proceed to study users who learn (Sec.~\ref{sec: strat users and learning}).
Our main result is an efficient algorithm for learning $h$
(Thm.~\ref{theorem: minimizing empirical error}),
which holds under a certain realizability assumption on $v$.
%  (Sec.~\ref{sec: strat users and learning}).
The algorithm minimizes the empirical error over a hypothesis class of set-function classifiers, whose complexity is controlled by parameter $k$,
thus allowing users to trade off expressivity and statsitical efficiency.
% (the degree of synergies---like complements---that can be taken into account when classifying the system's representation), \vin{assuming the user's true decisions based on their valuations is realizable (see Theorem \ref{theorem: minimizing empirical error})}.
The algorithm builds upon several structural properties which we establish for set-function classifiers.
Our key result here is that `folding' the system's strategic response
into the hypothesis class results in an induced class
having a simple form that makes it amenable to efficient optimization
% e.g., that it suffices to use simple-form set functions even when taking into account the system's strategic response
(Thm.~\ref{theorem: nice structure}). %For learning users 
%\nir{should say something about realizability here, perhaps that our algorithm is exact under a certain notion of `induced' realizability}

Building on this,
we continue to study the `balance of power' (Sec.~\ref{sec:balance}),
as manifested in the interplay between $k$
(hypothesis complexity, which reflects the power of the user)
and the range $[k_1,k_2]$ (which determines the power of the system).
For fixed $k_1,k_2$,
we analyze how the choice of $k$ affects the user's classification error,
through its decomposition into estimation and approximation errors.
% We analyze and bound both the estimation and approximation errors.
For estimation, we give a generalization bound (Thm.~\ref{theorem: generalization bound}),
obtained by analyzing the VC dimension of the induced function class
(as it relies on $k$).
For approximation, we give several results (e.g., Thm.~\ref{theorem: 0 approx error})
that link the expressivity of
the user's value function $v$ to the complexity of the learned $h$
(again, as it relies on $k$).
% (and so the expressivity of the learned $h$).
% for example, Thm.~\ref{theorem: 0 approx error} shows that when the complexity of $h$ matches that of the user's valuation, the approximation error vanishes.
% Together, our results portray how the choice of $k$ trades
Together, our results characterize how much is gained versus how much effort is invested in learning as $k$ is varied.
One conclusion is that even minimal learning can help significantly (whereas no learning can be catastrophic).

% Thus our main result is ``how to learn''; we also quantify how bad things can be without learning, and show that even minimal learning can significantly help.
 %For approx. 
%We give an algorithm that minimizes the empirical error. Separate the learning and the error analysis. balance of power -- how parameters of the problem affect the two sides. the system controls $k_1,k_2$ and the user controls $k$. These are the parameters of the game. 

From the system's perspective, and for fixed $k$,
we study how the range $[k_1,k_2]$ affects the system.
Intuitively, we would expect that increasing the range should be beneficial to the system, as it provides more alternatives to choose $z$ from.
% Our results show the following:
% %\begin{itemize}[leftmargin=*]
% %\item
% For fixed representation range $[k_1,k_2]$, %the user can control $k$, the complexity class of learned functions.
% \nir{rewrite using `fix system power, increase user power' etc}
% the larger the complexity $k$, the better the user's payoff, because learned functions are more expressive and can better approximate the valuation $v$. 
% Theorem~\ref{theorem: 0 approx error} shows that when the complexity matches that of the user's valuation, the approximation error vanishes.
% On the other hand, a large $k$ is costly in running time and in data---it means larger sample complexity and so larger estimation error (Theorem~\ref{theorem: generalization bound}). We thus get a tradeoff controlled by the user's effort level~$k$.  
% Now keep the complexity $k$ fixed and let the system control $k_1,k_2$. 
However, perhaps surprisingly, we find that the system can increase its payoff by `tying its hands' to a lower $k_2$.
This is because $k_2$ upper-bounds not only the system's range but also the `effective' $k$ of the user (who gets nothing from choosing $k>k_2$), and the lower the $k$, the better it is for the system (Lemma~\ref{thm:small_k2_is_better}).
The choice of $k_1$ turns out to be immaterial against fully strategic users, but highly consequential against users that are not.

\subsection{Relation to Strategic Classification}
Our formalization of strategic representation
shares a deep connection to strategic classification \citep{HardtMPW16}.
Strategic representation and strategic classification share an underlying
structure (a leading learning player who must take into account a strategic responder), but there are important differences. The first is conceptual: in our setting, roles are reversed---it is users who learn (and not the system),
and the system strategically responds (and not users).
% \footnote{Our notions of system and user mirror practice -- the system controls the platform where the interaction takes place, whereas the user is the (usually human) player using the platform.}
% E.g., in SC the system aims to learn an accurate classifier, whereas in our setting it is the user who learns. 
% The implications of this change are mainly to the balance of power between the sides (see Sec.~\ref{sec: balance of power}). 
This allows us to pursue questions regarding the susceptibility of \emph{users} to manipulation, with different emphases and goals,
while maintaining the `language' of strategic classification.

The second difference is substantive:
we consider inputs that are \emph{sets}
(rather than continuous vectors),
and manipulations that \emph{hide} information
(rather than alter it).
% (in the language of SC, we fix a particular cost function---distinct from the standard ones in the literature).
%different from distance-based ones in the literature).  %is distinct from the usual one costs in this literature). %what $c$ is (in sr it expresses truthfulness constraints). 
%The key difference between strategic representation and strategic classification lies in the object-type of inputs, and what accounts for a feasible modification. 
%In strategic classification, inputs $x$ are continuous vectors; arbitrary modifications are allowed, but these are costly, and the cost of modifying $x \mapsto x'$ is given by a cost function $c(x,x')$(e.g., $\|x-x'\|$).
%In contrast, in strategic representation, inputs $x$ are sets, and `modifications` $z$ must be subsets of constrained cardinality.\footnote{Note that both constraints can be expressed via a `hard' cost function: $c(x,z)=\infty$ if $z$ is infeasible, and 0 otherwise.}
% constrained to be truthful and cardinality-constrained;
% note that these can be expressed via a `hard' cost function
% $c(x,z)=\infty\{z \not\subseteq x \vee |z| \not\in [k_1,k_2]\}$.
Technically, switching to discrete inputs can be done by
utilizing the cost function to enforce truthfulness constraints
as a `hard' penalty on modifications.
But this change is not cosmetic:
since the system behaves strategically by optimizing over subsets,
learning must account for set-relations between different objects in input space;
this transforms the problem to one of learning set functions.\footnote{This 
is a subtle point: Since sets can be expressed as binary membership vectors,
it is technically possible to use conventional vector-based approaches to learn $h$.
Nonetheless, these approaches cannot account for strategic behavior;
this is since $\phi$ implements a set operation, which is ill-defined
for continuous vector inputs.}
%\vineet{Last two statements were not clear; probably we need to say what implications on learning.}
%\todo{strat behavior - need to reason about set relations}
From a modeling perspective,
subsets make our work compatible with classic attribute-based consumer decision theory \citep{lancaster1966new}.

% From a learning perspective,
% whereas strategic classification considers `conventional' learning of vector functions, strategic representation focuses on learning \emph{set functions}; 

Overall, we view the close connection to strategic classification as a strength of our formalization, showing that the framework of strategic classification is useful far beyond what was previously considered; and also that a fairly mild variation can lead to a completely different learning problem, with distinct  algorithmic and statistical challenges.
%Strategic representation can be thought of as a `reversal' of the problem of strategic classification, bearing key structural similarities, but also having several important differences. entity that controls the platform where the interactions happen.
%\inbal{1) Lots of recent papers on strategic classification. 2) We do the reverse. 3) There are 2 papers that look at the reversed -- zrnic. conitzer. 

\subsection{Related Work \hfill \normalfont{\textit{(see also Appx. \ref{secapp: additional related work})}}} \label{subsec: rel work}

\paragraph{Strategic classification.}
Strategic classification is a highly active area of research.
Recent works in this field include
statistical learning characterizations
\cite{zhang2021incentive,sundaram2021pac,ghalme2021strategic},
practical optimization methods
\cite{LevanonR21,levanon2022generalized},
relaxation of key assumptions
\cite{ghalme2021strategic,BechavodPZWZ21,jagadeesan2021alternative,
levanon2022generalized,eilat2022strategic},
relations to causal aspects
\cite{miller2020strategic,chen2020linear},
and consideration of societal implications
\cite{MilliMDH19,HuIV19,chen2020strategic,LevanonR21}.

Two recent works are closely relate to ours:
\citet{zrnic2021leads} %zrnic looks at strategic classification over time. they have repeated play (things don't converge after 1 step) and the rate means who is first, who is making the first move. \emph{theirs 
consider a dynamic setting of repeated learning that can change the \emph{order} of play; in contrast, we switch the \emph{roles}.
\citet{krishnaswamy2021classification} consider information withholding by users (rather than the system).
They aim to learn a truth-eliciting mechanism, %, enforced as a constraint,
which incentivizes the second player to reveal all information (i.e., `the whole truth').
% They focus on masking attributes by the user (rather than dropping attributes by the system),
% Hence, they aim for classifiers that incentivize ;
Their mechanism ensures that withholding never occurs;
in contrast, our goal is to predict \emph{despite} strategic withholding
(i.e., `nothing but the truth’).
\textbf{Bayesian persuasion.}
%The economic literature studies how a well-informed system can use its information advantage to influence its users. We are not aware 
%not aware of work in ML that studies system manipulating the user. %information shrouding milgrom paper -- all econ. In econ there are a lot of works studying system manipulating but no learning.
In Bayesian persuasion~\citep{KamenicaGentzkow11}, %(see also e.g.~\cite{dughmi2019algorithmic,zu2021} for an algorithmic/learning approach). 
%Here too 
a more-informed player (i.e., the system) uses its information advantage coupled with commitment power to influence the choices of a decision maker (i.e., the user). 
Works closest to ours are by~\citet{DughmiIOT15}, who upper bound the number of signaled attributes, and by~\citet{haghtalab2021persuading}, who study %representations of reality through 
strategic selection of anecdotes.
Both works consider the human player as a learner (as do we).
However, in our work, the order of play is reversed---the decision-maker (user) moves first and the more-informed player (system) follows.
%whose valuation is known. %of the receiver must be known and the receiver must be capable of Bayesian inderenf
%strategically signals disclosure of information to influence decision-making has long been studied in economics
%The more-informed player moves first, committing to a scheme for signaling information (e.g., recommending actions) to the decision-maker. %o the decision-maker. 
% This fits a scenario in which
Bayesian persuasion also assumes that the system knows the user's valuation,
%can recommend to her persuasively whether to consume the item based on their 
and crucially relies on both parties knowing the distribution $D$ of items.
% ---this being crucial for persuasion to be successful.
In contrast, we model the user as having only sample access to $D$, and the system as agnostic to it.

\section{A Formalization of Strategic Representation}\label{sec: setting}

% \todo{make it clear at some point that we focus on a single user}

We begin by describing the setting from the perspective of the user,
which is the learning entity in our setup.
We then present the `types' of users we study,
and draw connections between our setting and others
found in the literature.

\subsection{Learning Setting}
In our setup, a user is faced with a stream of items,
and must choose which of these to consume.
Items are discrete, with each item $x \in \X \subseteq 2^E$
described by a subset of ground attributes, $E$,
where $|E|=q$.
We assume all feasible items have at most $n$ attributes, $|x| \le n$. 
%\inbal{this seems different than the submitted version?}.
The value of items for the user are encoded by a
\emph{value function}, $v:\X \rightarrow \R$.
We say an item $x$ is \emph{worthwhile} to the user
if it has positive value, $v(x)>0$,
and use $y=Y(x)=\sign(v(x))$ to denote worthwhileness,
i.e., $y=1$ if $x$ is worthwhile, and $y=-1$ if it is not.
Items are presented to the user as samples drawn i.i.d.~from some unknown distribution $D$ over $\X$,
% determined by the underlying data generating process.
and for each item,
the user must choose whether to consume it
(e.g., click, buy, watch) or not.
% Broadly speaking, the user seeks to choose items
% if and only if they are worthwhile.

We assume that the user makes choices regarding items
by committing at the onset to a \emph{choice function}
$h$ that governs her choice behavior.
In principle, the user is free to chose $h$
from some predefined function class $H$;
learning will consider finding a good $h \in H$,
but the implications of the choice of $H$ itself
will play a central role in our analysis.
Ideally, the user would like to choose items if and only if
they are worthwhile to her;
practically, her goal is to find an $h$ for which this
holds with large probability over $D$.
For this,
the user can make use of her knowledge regarding items she has
already consumed, and therefore also knows their value;
we model this as providing the user access to a
labeled sample set $S=\{(x_i,y_i)\}_{i=1}^m$
where $x_i \sim D$ and $y_i=\sign(v(x_i))$, %\footnote{\blue{Our results also hold for noisy $y$.}}
which she can use for learning $h$.

\textbf{Strategic representations.}
The difficulty in learning $h$ is that
user choices at test time must rely only on
item \emph{representations}, denoted $z \in \Z$,
rather than on full item descriptions.
Thus, learning considers choice functions that operate on representations,
$h:\Z \rightarrow \pmone$;
the challenge lies in that while
choices must be made on the basis of representations $z$,
item values are derived from their full descriptions $x$---which
representations describe only partially.

The crucial aspect of our setup is that representations are not arbitrary;
rather, \emph{representations are controlled by the system},
which can choose them strategically to promote its own goals.
We model the system as acting through a \emph{representation mapping}, $\rep: \X \rightarrow \Z$,
which operates independently on any $x$,
and can be determined in response to the user's choice of $h$.
This mimics a setting in which a fast-acting system
can infer and quickly respond to a user's (relatively fixed) choice patterns.

We assume the system's goal is to choose a $\rep_h$ that maximizes \emph{expected user engagement}:
\begin{equation}
\label{eq:system_objective}
% \text{\textbf{Engagement}:} \qquad
\expect{x\sim D}{\one{h(\rep_h(x))=1}}.
% \prob{x\sim D}{h(\rep_h(x))=1}
\end{equation}
Nonetheless, representations cannot be arbitrary,
and we require $\rep_h$ to satisfy two properties.
First, chosen representations must be \emph{truthful},
meaning that $z \subseteq x$ for all $x$.
Second, representations are subject to \emph{cardinality constraints},
$k_1 \le |z| \le k_2$ for some predetermined $k_1,k_2 \in \N$.
We will henceforth use $\Z$ to mean representations of feasible cardinality.
Both requirements stem from realistic considerations:
A nontruthful system which intentionally distorts item information
% which intentionally presents faulty information regarding items
is unlikely to be commercially successful in the long run;
intuitively, truthfulness gives users some hope of resilience to manipulation.
For $k_1,k_2$, we think of these as exogenous parameters of the environment,
arising naturally due to physical restrictions
(e.g., screen size) or cognitive considerations (e.g., information processing capacity);
if $k_2 < n$, we say representations are \emph{lossy}.

Under these constraints,
the system can optimize Eq. \eqref{eq:system_objective}
by choosing representations via the \emph{best-response mapping}:
\begin{equation}
\label{eq:br_rep}
\rep_h(x) = \argmax_{z \in \Z} h(z)
\,\,\,\, \text{ s.t. } \,\,\,\,
z \subseteq x,  |z| \in [k_1,k_2]
\end{equation}
Eq. \eqref{eq:br_rep} is a best-response since it maximizes Eq. \eqref{eq:system_objective} for any given $h$:
for every $x$, $\rep_h$ chooses a feasible $z \subseteq x$
that triggers a positive choice event, $h(z)=1$---if such a $z$ exists.
In this way, $k_1,k_2$ control how much leeway the system has in revealing only partial truths;
as we will show, both parameters play a key role in determining outcomes
for both system and user.
From now on we overload notation and by $\rep_h(x)$ refer to this best-response mapping.  

% hence, $\Z = \{z \subseteq x \,:\, x \in \X$.

% Note that while $h$ can be learned based on full item descriptions $x$,

\textbf{Learning objective.}
Given function class $H$ and a labeled sample set $S$,
the user aims to find a choice function $h \in H$
that correctly identifies worthwhile items given their representation,
and in a way that is robust to strategic system manipulation.
The user's objective is therefore to maximize:
\begin{equation}
\label{eq:user_objective}
% \hhat = \argmin_{h \in H}
\expect{x \sim D}{\one{h(\rep_h(x)) = y}}
\end{equation}
where $\rep_h$ is the best-response mapping in Eq. \eqref{eq:br_rep}. 
%\inbal{Should we say something about the tension between the learning objective and the game-theoretic objective of maximizing the utility from the items consumed?}

Note that since $h$ is binary, the argmax of $\rep_h$ may not be unique;
e.g., if some $z_1\subseteq x,z_2\subseteq x$ both have $h(z_1)=h(z_2)=1$.
Nonetheless, the particular choice of $z$ does not matter---from the user's perspective,
her choice of $h$
is invariant to the system's choice of best-response $z$ (proof in Appx. \ref{secapp: setting proof}):
\begin{restatable}{observation}{samePayoff}
\label{lem:br_equiv}
Every best-response $z \in \phi_h(x)$ 
induces the same value in the user's objective function
(Eq. \eqref{eq:user_objective}).
% (proof in Appx. \ref{secapp: setting proof}).
\end{restatable}

\subsection{User Types}
Our main focus throughout the paper will be on users that
learn $h$ by optimizing Eq. \eqref{eq:user_objective}.
But to understand the potential benefit of learning,
we also analyze `simpler' types of user behavior.
Overall, we study three user types, varying in their sophistication
and the amount of effort they invest in choosing $h$.
These include:
\begin{itemize}
\item 
\textbf{The \naive\ user:} Acts under the (false) belief that
representations are chosen in her own best interest.
This user type truthfully reports her preferences to the system
by setting $h=v$ as her choice function.\footnote{Note that while $h$ takes values in $\X$,
$v$ takes values in $\X$. Nonetheless, truthfulness implies that $\Z \subseteq \X$,
and so $v$ is well-defined as a choice function over $\Z$.}

\item
\textbf{The agnostic user:} Makes no assumptions about the system.
This user employs a simple strategy that relies on basic data statistics
which provides minimal but robust guarantees regarding her payoff.

\item
\textbf{The strategic user:} Knows that the system is strategic,
and anticipates it to best-respond.
This user is willing to invest effort (in terms of data and compute)
in learning a choice function $h$ that maximizes her payoff
by accounting for the system's strategic behavior.
% \nir{where should we say this users is only boundedly/partially strategic?}

\end{itemize}

Our primary goal is to study the balance of power between
users (that choose) and the system (which represents).
In particular, we will be interested in exploring the 
tradeoff between a user's effort and her susceptibility to manipulation.
% We begin with the \naive\ and agnostic users,
% and then turn to the strategic learning user.

\subsection{Strategic Representation as a Game}
Before proceeding, we give an equivalent characterization
of strategic representation as a game.
Our setting can be compactly described as a single-step Stackelberg game:
the first player is User,
which observes samples $S=\{(x_i,y_i)\}_{i=1}^m$,
and commits to a choice function $h:\Z \rightarrow \pmone$;
the second player is System, which given $h$,
chooses a truthful $\phi_h:\X \rightarrow \Z$
(note how $\phi_h$ depends on $h$).
The payoffs are:
\begin{align}
& \text{\textbf{User:}} \qquad \,\, \expect{x \sim D}{\one{h(\phi_h(x))=y}} \label{eq:user_payoff} \\
& \text{\textbf{System:}} \quad \,\, \expect{x \sim D}{\one{h(\phi_h(x))=1}} \label{eq:system_payoff}
\end{align}
%
% \paragraph{Relation to strategic classification.}
Note that payoffs differ only in that User seeks \emph{correct} choices, whereas System benefits from \emph{positive} choices.
This reveals a clear connection to strategic classification,
in which System, who plays first,
is interested in \emph{accurate} predictions,
and for this it can learn a classifier;
and User, who plays second,
can manipulate individual inputs (at some cost)
to obtain \emph{positive} predictions.
Thus, strategic representation can be viewed as a
variation on strategic classification, but with roles `reversed'.
Nonetheless, and despite these structural similarities,
strategic representation bears key differences:
% two key differences are that in strategic representation
items are discrete (rather than continuous),
manipulations are subject to `hard' set constraints
(rather than `soft', continuous costs),
and learning regards set functions (rather than vector functions).
These differences lead to distinct questions
and unique challenges in learning.

\section{Warm-up: \Naive\ and Agnostic Users}
%\todo{rethink section name}
\label{sec:Naive}

%We begin with analyzing two user types: \naive, and agnostic. 

%\subsection{
\textbf{The \naive\ user.}
The \emph{\naive\ user}
employs a `what you see is what you get' policy:
% plays according to `face value';
given a representation of an item, $z$,
this user estimates the item's value based on $z$ alone,
acting `as if' $z$ were the item itself.
Consequently, the \naive\ user sets $h(z)=\sign(v(z))$,
even though $v$ is truly a function of $x$.
The \naive\ user fails to account for the system's strategic behavior
(let alone the fact that $z \subseteq x$ of some actual $x$).

% As noted earlier, a \naive\ user (also referred to as a myopic user \cite{gabaix2006shrouded}) does not reason about the missing information and takes decisions  based purely on the revealed information.  This   behavior can be attributed to ignorance,  high cognitive costs for inference, or the belief that the system is fully transparent. A \naive\ user,  playing according to `face value', assumes (wrongly) that `what you see is what you get' and hence on  seeing $z$
% \blue{does not consider that there is an $x \supseteq z$ from which value derives. Consequently a \naive\ user on  seeing $z$ plays $v(z)$ instead of $v(x)$.}
%, plays  $v(z)$ instead of  $v(x)$
%for whatever $x \supseteq z$ is.  %That is, a  \naive\  User's choice function is given as  $h(z) = v(z)$ for all $z \in \mathcal{Z}$. 

%\outline{result: if the system is benevolent, then things work out nicely, that is, user gets most of the payoff}\\

Despite its naivety, there are conditions under which this user's approach makes sense. Our first result shows that the \naive\ policy is sensible
in settings where the system is \emph{benevolent}, and promotes user
interests instead of its own.
% \blue{We begin the following result which shows that for a \naive user the payoff of the system is maximized if the system is \emph{benevolent}.
\begin{restatable}{lemma}{NaiveUser}
%Assume $v$ is such that for each $x\in \X$, there exists a $z\in \Z$ and $z\subseteq x$ such that $v(z) = v(x)$. Then 
If system plays the \emph{benevolent} strategy:
\[
\phi^{\mathrm{benev}}_h(x) = \argmax_{z \subseteq x, |z| \in [k_1,k2]}\{\mathbbm{1}\{h(z) = \sign(v(x))\},
\]
then the \naive\ approach maximizes user payoff.
% \blue{against any truthful system}.
% The payoff of a \naive\ user is maximized if the system responds benevolently, that is, the system's response for each $x$, $\phi(x) = \argmax_{z\in \Z, z\subseteq x}\{\mathbbm{1}\{v(z) = v(x)\}$.
\end{restatable}
Proof in Appx. \ref{secapp: learning proofs}.
The above lemma  is not meant to imply that \naive\ users \emph{assume}
the system is benevolent;
rather, it justifies why users having this belief might act in this way.
% }
%Next, we give a simple example where a \naive  
%\nir{note: user does not need to `assume' benevolence holds (or anything really); rather, if this is the state of the world, then things would be fine)}
Real systems, however, are unlikely to be benevolent;
our next example shows a strategic system can 
easily manipulate \naive\ users to receive arbitrarily low payoff.

\begin{example}
\label{ex:one}
Let $x_1 = \{a_1\}, x_2 = \{a_1,a_2\}, x_3 = \{a_2\}$ with $v(x_1) =  +1$ and $v(x_2)=v(x_3)= -1$.
Fix $k_1=k_2=1$, and let $D  = (\varepsilon/2,1- \varepsilon, \varepsilon/2)$.
Note $\Z=\{a_1,a_2\}$ are the feasible representations.
The \naive\  user assigns $h=(a_1)=+1,h(a_2)=-1$ according to $v$.
For $x_2$, a strategic system plays $\phi(x_2) = a_1$.
The expected payoff to the user is $\varepsilon$.
\end{example}

One reason a \naive\ user is susceptible to manipulation
is because she does not make any use of the data she may have.
We next describe a slightly sophisticated user that uses a simple strategy to ensure a better payoff.

\textbf{The agnostic user.}
%\outline{show that with minimal effort users can be better off}
The \emph{agnostic user} puts all faith in data;
this user does not make assumptions on, nor is she susceptible to, 
the type of system she plays against.
Her strategy is simple: collect data, compute summary statistics,
and choose to either always accept or always reject %system's suggestions
(or flip a coin).
In particular, given a sample set $S=\{(x_i,y_i)\}_{i=1}^m$,
the agnostic user first computes the fraction of positive examples,
$\muhat :=\frac{1}{m}\sum_{i=1}^{m} y_i$.
Then, for some tolerance $\tau$, sets for \emph{all} $z$,
$h(z)=1$ if $\muhat  \geq 1/2+\tau$,
$h(z)=-1$ if $\muhat \leq 1/2-\tau$,
and flips a coin otherwise.
In Example \ref{ex:one}, an agnostic  user would    choose $h=(-1, -1)$  when $m$ is large,  guaranteeing a payoff of at least  $\frac{\sqrt{m}(1-\varepsilon/2)}{2+\sqrt{m}}\rightarrow (1-\varepsilon/2)$ as $m \rightarrow \infty$. Investing minimal effort, 
for an appropriate choice of $\tau$,
this user's strategy turns out to be quite robust.
\begin{theorem}\label{theorem: agnostic informal}
(Informal)
Let $\mu$ be the true rate of positive examples, $\mu = \expect{D}{Y}$.
Then as $m$ increases, the agnostic user's payoff approaches $\max\{\mu,1-\mu\}$
at rate $1/\sqrt{m}$.
\end{theorem}
Formal statement and proof in Appx. \ref{secapp: agnostic}.
In essence, the agnostic user guarantee herself the `majority' rate
with rudimentary usage of her data, and in a way that does not depend on how system responds.
But this can be far from optimal;
we now turn to the more elaborate \emph{strategic user}
who makes more clever use of the data at her disposal.

\section{Strategic Users Who Learn}
\label{sec: strat users and learning}
A strategic agent acknowledges that the system is strategic,
and anticipates that representations are chosen to maximize her own engagement.
Knowing this, the strategic user makes use of her previous experiences,
in the form of a labeled data set $S=\{(x_i,y_i)\}_{i=1}^m$,
to learn a choice function $\hhat$ from some function class $H$
that optimizes her payoff (given that the system is strategic).
Cast as a learning problem, this is equivalent to 
minimizing the expected classification error on strategically-chosen representations:
\begin{equation}
\label{eq:exp_objective}
h^* = \argmin_{h \in H} \expect{D}{\one{h(\rep_h(x)) \neq y}}.
\end{equation}
Since the distribution $D$ is unknown, 
we follow the conventional approach of empirical risk minimization (ERM)
and optimize the empirical analog of Eq.~\eqref{eq:exp_objective}:

\begin{equation}
\label{eq:emp_objective}
\hhat = \argmin_{h \in H} \frac{1}{m}
\sum_{i=1}^m h(\rep_h(x_i)) \neq y_i).
\end{equation}

\noindent Importantly since every $z_i=\phi_h(x_i)$ is a set,
$H$ must include \emph{set functions} $h:\Z \rightarrow \pmone$,
and any algorithm for optimizing Eq.~\eqref{eq:emp_objective}
must take this into account.
In Sections \ref{subsec: complexity of set functions} and \ref{subsec: learning via reduction to induced functions}, we characterize the complexity of a user's choice function and relate its complexity to that of $v$, and in Section \ref{sec:algo} give an algorithm that computes $\hhat$, the empirical minimizer, for a hypothesis class of a given complexity.

\subsection{Complexity Classes of Set Functions}\label{subsec: complexity of set functions}
Ideally, a learning algorithm should permit flexibility
in choosing the complexity of the class of functions it learns
(e.g., the degree of a polynomial kernel, the number of layers in a neural network),
as this provides means to trade-off running time with performance
and to reduce overfitting.
In this section we propose a hierarchy of set-function complexity classes that is appropriate for our problem.

Denote by $\sbstk_k(z)$
all subsets of $z$ having size at most $k$:
\[
\sbstk_k(z) = \{ z' \in 2^E \,:\, z'\subseteq z, |z'| \le k\}.
\]
We start by defining $k$-order functions over the representation space. These functions are completely determined by weights placed on subsets of size at most $k$.
% (see also~\citet{conitzer2005combinatorial} and Sec. \ref{subsec: rel work}).
%\inbal{what's the exact relation to conitzer et al? also, i suggest to add a sentence that says in words what's about to come in the definition}.

\begin{definition}
\label{def:hyper}
We say the function $h : \Z \rightarrow \pmone$
is of \emph{order $k$} if there exists real
% basis \inbal{basis weights are not defined} 
weights on sets of cardinality at most $k$, 
$\{w(z') \,:\, z' \in \sbstk_k(z)\}$,
such that 
\[
h(z) = \sign \left( \sum\nolimits_{z' \in \sbstk_k(z)} w(z') \right).
\]
\end{definition}
Not all functions $h(z)$
can necessarily be expressed as a $k$-order function (for some $k$);
% In principle, $k_2$ restrict us to working with $k$-order functions having $k \le k_2$.
% \vin{It is possible that a function $h$ may not be a $k$-order function for any $k\leq k_2$. But}
nonetheless, in the context of optimizing Eq.~\eqref{eq:emp_objective},
we show that working with $k$-order functions is sufficiently general, since
any set function $h$ can be linked to a matching
$k$-order function $h'$ (for some $k \le k_2$)
through how it operates on strategic inputs.

\begin{restatable}{lemma}{orderkfunctions}
\label{lemma: H consists of order k functions}
For any $h:\Z \rightarrow \pmone$, there exists $k \le k_2$
and a corresponding $k$-order function $h'$ such that:
\[
h(\rep_h(x)) = h'(\rep_{h'}(x)) \,.
\]
\end{restatable}
% Proof in Appx. \ref{secapp: strategic users who learn}.
% \begin{lemma}
% Corresponding to any $h:\Z \rightarrow \pmone$ there is an equivalent $h':\Z \rightarrow \pmone$, a $k \leq k_2]$ and a $w: \Z_{\leq k} \rightarrow \mathbb{R}$ such that a) $h'(z) = \sign(w(z))$ for all $|z| \leq k$ and $z\in \Z$, b) 
% \[
% h'(z) = \sign \left( \sum\nolimits_{z':z' \subseteq z, |z'|\leq k} w(z') \right)
% \]
% c) $h(\phi_h(x) = h'(\phi_{h'}(x))$ for all $x\in \X$.
% \end{lemma}
% \todo{Write why the choice of weights defining $h'$ in the proof of Lemma \ref{lemma: H consists of order k functions} motivates us to look at $H_k$, which are special type of order $k$ functions}
Lem. \ref{lemma: H consists of order k functions}
% we show that for learning purposes,
% any set function $h$ can be linked to a matching
% $k$-order function $h'$ (for some $k$)
% through how it operates on strategic inputs.
permits us to focus on $k$-order functions.
The proof %of Lem. \ref{lemma: H consists of order k functions}
is constructive (see Appx. \ref{secapp: strategic users who learn}), and the construction itself turns out to be
highly useful.
In particular, the proof constructs $h'$ having a particular form of
\emph{binary} basis weights, 
$w(z) \in \{a_-,a_+\}$,
% for some choice of
% $a_- < -1$ and $a_+ > \sum_{i \in [k]} {n \choose k}$
which we assume from now on are fixed ($\forall k$).
Hence, every function $h$ has a corresponding binary-weighted
$k$-order function $h'$,
which motivates the following
definition of functions and function classes.
\begin{definition} \label{def:k-order}
% Fix some $a_- < -1$ and $a_+ > \sum_{i \in [k]} {n \choose k}$.
We say a $k$-order function $h$ with basis weights $\w$
is \emph{binary-weighted} if:
\begin{equation*}
w(z) 
    \begin{cases}
        \in \{a_-, a_+\} & \forall z \text{ such that } |z| = k \\
        = a_- & \forall z \text{ such that } |z| < k
    \end{cases}
\end{equation*}
for some fixed
$a_- \in (-1,0)$ and $a_+ > \sum_{i \in [k]} {n \choose i}$.
\end{definition}
% \ganesh{ Given constraint  on the size of the representation, the above described binary weighted function  fully  encodes  the optimal representation strategy of the system as follows. }
% From this basis representation, we can define a natural measure of
% complexity and a hierarchy of nested classes.
% \vin{Binary-weighted functions can equivalently be thought of as follows: There is a family of subsets that guarantee acceptance, and the function accepts every set with a subset from that family. If it’s $k$-order then the family should have subsets of size $k$.}
% \nir{explain using $a_-,a_+$}
A binary weighted $k$-order $h$ determines a family of $k$-size subsets, described by having weights as $a_+$, such that for any $z\in \Z$ with $|z| \ge k$, $h(z)=1$ if and only if $z$ contains a subset from the family (for $z$ with $|z|<k$, $h(z)=-1$ always).
This is made precise using the notion of \emph{lifted functions} in Lem. \ref{lemma: h in H_k are lifts on size at most k} in Appx. \ref{secapp: lifted functions}.
%here is $h$, and the $w$ corresponding to it. What this ensures is that the family of subsets given weights $a_+$ have $h(z) = 1$ and hence, guarantee acceptance for every $x$ containing such a subset.}
Next, denote: % the class of binary-weighted $k$-order functions as:
\begin{equation*}
H_k = \{h \,:\, h \text{ is a binary-weighted $k$-order function}\}. 
\end{equation*}
The classes $\{H_k\}_k$ will serve as complexity classes for our learning algorithm;
the user provides $k$ as input,
and \textsc{Alg} outputs an $\hhat \in H_k$ that minimizes the empirical loss\footnote{Assuming the empirical error is zero.}.
% \nir{the above is not true. the algorithm outputs $\hhat \in H_k$, which are binary-weighted order-$k$ - so $\hhat$ is not optimal out of all order $k$ functions (ie, there could be a better order-$k$ function having non-binary weights}
As we will show,
using $k$ as a complexity measure provides
% this provides
the user direct control over the tradeoff between estimation and approximation error,
as well as over the running time.

Next, we show that the $\{H_k\}_k$ classes are strictly nested.
This will be important for our analysis of approximation error,
as it will let us reason about the connection between
the learned $\hhat$ and the target function $v$ (proof in Appx. \ref{secapp: strategic users who learn}).
\begin{restatable}{lemma}{stricthierarchy}
For all $k$, $H_{k-1} \subseteq H_{k}$
and $H_{k} \setminus H_{k-1} \neq \emptyset$.
% there is an $h \in H_k$ but $h \not\in H_{k'}$ for $k'<k$.
\end{restatable} 
%%%%%%%%%%%%%Comment%%%%%%%%%%%%%%%%%%%%
\begin{comment}
Denote $Z_\ell = \{z \in 2^E : |z| =\ell\}$,
and note that all feasible representations can be partitioned as
$\Z = Z_{k_1} \uplus \ldots \uplus Z_{k_2}$.
We refer to functions that operate on single-sized sets 
as \emph{restricted functions}.
Our next result shows that choice functions in $H_k$
can be represented by restricted functions over $Z_k$ that are `lifted' to operate on the entire $\Z$ space.
This will allow us to work only with sets of size exactly $k$ in the algorithm given in Sec. \ref{sec:algo}.
\begin{restatable}{lemma}{hklifts}\label{lemma: h in H_k are lifts on size at most k}
For each $h \in H_k$ with weights $\w$ 
there exists a corresponding $g: Z_{k} \rightarrow \pmone$
such that $h = \lift(g)$, where:
\begin{equation*}
\lift(g)(z) =
    \begin{cases}
        1 & \text{\textnormal{if }} k \le |z| \text{\textnormal{ and }} \\
        & \exists z' \subseteq z, |z'|=k \,\text{\textnormal{ s.t. }}\, g(z')=1 \\
        -1 & \text{\textnormal{o.w.}}
    \end{cases}
\end{equation*}
Further, for all $z \in Z_k$ $g(z) = 1$ if $w(z) = a_+ >0$, and $g(z) = -1$ otherwise.
\end{restatable}
\end{comment}
%%%%%%%%%%%%%%%End Comment%%%%%%%%%%%%%%%%%%%%%%%%%%%
Note that $H_n$ includes all binary-weighted set functions,
but since representations are of size at most $k_2$,
it suffices to consider only $k \le k_2$.
Importantly, $k$ can be set lower than $k_1$;
for example, $H_1$ is the class of threshold modular functions,
and $H_2$ is the class of threshold pairwise functions.
The functions we consider are parameterized by their weights, $\w$, and so any $k$-order function has at most
$|\w| = \sum_{i=0}^k{q \choose i}$ weights.
In this sense, the choice of $k$ is highly meaningful.
Now that we have defined our complexity classes,
we turn to discussing how they can be optimized over.

% \todo{maybe apx and est errs should go in a later section}

% \blue{
% \textbf{main tradeoff} - approximation-estimation: the lower the $k$:
% - the less expressive $H_k$, so further from $v$, so higher loss (approx error)
% - the lower the VC, so better generalization (est error)
% - the lower the computational cost
% }

% - the lower the VC, so better generalization (est error)
% - the lower the computational cost
% - the less expressive $H_k$, so further from $v$, so higher loss (approx error)

% known: (?!)
% any scalar set function $g:\Z \rightarrow \R$ can be expressed as follows:
% $g(z) = \sum_{z' \subseteq z} v(z')$ for some weights $v$

% Since representations are constrained in their cardinality,
% $H$ must include functions $h$ that can operate on
% any $z \in \Z$ with $k_1 \le |z| \le k_2$.
% These, however, can be partitioned by their
% \emph{maximal effective cardinality}.

% users chooses $H$ having functions $h(z)$ for 
% Complexity is controlled by the \emph{effective} max size of representations, denoted $k$.
% Use lifting to operate on any $|z|$ between $k_1,k_2$
% Note $k \le k_2$ (will never see $|z| > k_2$), but can have $k < k_1$ (just less expressive)

% Note that $H_{k'} \subseteq H_k$ for $k'< k$
% $H = H_1 \cup \dots \cup H_{k_2}$
% $H_k = \{ h:\Z_k \rightarrow \pmone$

%------------------------------------------------------------

\subsection{Learning via Reduction to Induced Functions}\label{subsec: learning via reduction to induced functions}
The simple structure of functions in $H_k$
makes them good candidates for optimization.
But the main difficulty in optimizing the empirical error in
Eq.~\eqref{eq:emp_objective} is that the choice of $h$
does not only determine the error,
but also determines the inputs on which errors are measured
(indirectly through the dependence of $\rep_h$ on $h$).
To cope with this challenge,
our approach is to work with \emph{induced} functions
that already have the system's strategic response
encoded within,
which will prove useful for learning.
Additionally, as they
operate directly on $x$ (and not $z$),
they can easily be compared with $v$,
which will become important in Sec.~\ref{sec: balance of power}.

%\paragraph{Induced functions.}
% We now turn to defining \emph{induced functions}
% and their corresponding classes.
\begin{definition} %[Induced functions]
\label{def:induced_class}
For a class $H$, its \emph{induced class} is:
\[
F_H \triangleq
\{ f:\X \rightarrow \pmone \,:\, \exists h \in H \text{ s.t. }
f(x) = h(\phi_h(x)) \}
\]
\end{definition}
The induced class $F_H$ includes for every $h \in H$
a corresponding function that already has $\rep_h$ integrated in it.
We use $F_k=F_{H_k}$ to denote the induced class of $H_k$.
For every $h$, we denote its induced function by $f_h$.
% the $f \in F$ corresponding to $h \in H$,
Whereas $h$ functions operate on $z$,
induced functions operate directly on $x$,
with each $f_h$ accounting internally for how the system strategic responds to $h$ on each $x$.

Our next theorem provides a key structural result:
induced functions inherit the weights
of their $k$-order counterparts.
%\outline{9. [RESULT]: $\ell$-set modular functions can be represented as threshold unit weighted sums with binary weights}

\begin{restatable}{theorem}{nicestructure}\label{theorem: nice structure}
% Let $k\leq k_2$.
For any $h \in H_k$ with weights $\w$:
% = \{w(z') \,:\, z' \in \sbstk_k(z)\}$ be the basis weights of cardinality at most $k$ such that 
\[
h(z) = \sign \left( \sum\nolimits_{z' \in \sbstk_k(z)} w(z') \right),
\]
its induced $f_h \in F_k$
can be expressed using the same weights, $\w$,
but with summation over subsets of $x$, i.e.:
$$
f_h(x) = \sign \left( \sum\nolimits_{z \in \sbstk_k(x)} w(z) \right).
$$
\end{restatable}
Thm.~\ref{theorem: nice structure} is the main pillar on which our algorithm stands:
it allows us to construct $h$ by querying the loss
\emph{directly}---i.e., without explicitly computing $\rep_h$---by
working with the induced $f_h$;
this is since:
\[
\one{h(\rep_h(x_i)) \neq y_i} = \one{f_h(x_i) \neq y_i}
\]
Thus, through their shared weights, induced functions
serve as a bridge between \emph{what} we optimize,
and \emph{how}.

%------------------------------------------------------------

\subsection{Learning Algorithm} \label{sec:algo}

% \outline{ algorithm 1: given sample set $S$ and choice of $k$,
% find optimal $h \in H_k$}

% \nir{can we tidy up the algorithm? would be good if we can make the pseudocode easier to follow (and shorter!)}

We now present our learning algorithm,  \textsc{Alg}.
The algorithm is exact: it takes as input a training set $S$ and a parameter $k$,
and returns an $h \in H_k$ that minimizes the empirical loss
(Eq. \eqref{eq:emp_objective}).
Correctness holds under the realizabilility condition
$Y\in F_k$, i.e., $Y$ is the induced function of some $h \in H_k$.\footnote{Note that even for standard linear binary classification,
finding an empirical minimizer of the $0/1$ in the agnostic (i.e., non-realizable) case is NP-hard \citep{shalev2014understanding}.}
%it takes as input a training set $S$ and a parameter $k$,
%and returns an $h \in H_k$ that minimizes the empirical loss
%(Eq. \eqref{eq:emp_objective}) if $Y \in F_k$.

The algorithm constructs $h$ by sequentially
computing its weights, $\w=\{w(z)\}_{|z| \le k}$.
As per Def.~\ref{def:k-order}, only $w(z)$ for $z$ with $|z|=k$ must be learned;
hence, weights are sparse,
in the sense that only a small subset of them are assigned $a_+$,
while the rest are $a_-$.
Weights can be implemented as a hash table,
where $w(z)=a_+$ if $z$ is in the table,
and $w(z)=a_-$ if it is not.
% We prove correctness in Thm. \ref{theorem: minimizing empirical error}. 
Our next result establishes the correctness of \textsc{Alg}.
The proof leverages a property that characterizes the existence of an $h\in H_k$ having zero empirical error
(see Lem. \ref{lem: 0 error charac}). 
%The proof of Lem. \ref{lem: 0 error charac} uses Lem. \ref{lemma: h in H_k are lifts on size at most k} and Thm. \ref{theorem: nice structure}; Lem. \ref{lemma: h in H_k are lifts on size at most k} is used to connect functions in $H_k$ with functions operating over size $k$
% subsets via the lifting operation.
The proof of Lem. \ref{lem: 0 error charac} uses Thm. \ref{theorem: nice structure}, which enables the loss to be directly computed for the induced functions using the shared weight structure.

% ---------------------------------End---------------------------
%\begin{center}
\begin{algorithm}[t!]\label{alg:algo}
\begin{algorithmic}[1]

    %\SetAlgoNoLine
    %\renewcommand{\thealgocf}{}
    %\DontPrintSemicolon
    %\SetAlgorithmName{$\textsc{Alg}$}{ }{ }
    \caption{ $\textsc{Alg}$}
% \textbf{Output:} User optimal choice function $h$
% \;

\STATE \textbf{Input}: $S = \{(x_i,y_i)\}_{i \in [m]}$, $k \in [k_2]$

\STATE \textbf{Pre-compute:} 
\STATE $ S^{+} = \{x\in S : y = +1\}$,
\STATE $S^{-} = \{x\in S : y = -1\}$ 
\STATE $Z_{k,S} = \{ z \,:\, |z|=k, \exists x\in S ~  z\subseteq x\}$ 
\STATE $\phat(x_i) = \frac{1}{m}\sum_{j\in [m]} \one{x_i = x_j}\quad  \forall  i \in [m]$ 
\STATE Fix $a_- \in (-1,0)$ and $a_+ > \sum_{i \in [1,k]} {n \choose i}$ 

\STATE \textbf{Initialize:} 
% $w: Z_S^{(k)} \rightarrow \mathbb{R}$, as $w(z) = 0$ for all $z \in Z_S^{(k)}$, \\
\STATE $Z^+ = \varnothing$, $Z^- = \varnothing$, $V = \varnothing$, \, 
$S_z = \varnothing \quad \forall z\in Z_{k,S}$ 

% In the algorithm below, we set $w(z)$ as either $a^+ > \sum_{i \in [1,k]} {n \choose i}$ or $a^- \in (0,1)$.
   %$\boldsymbol{u} = \boldsymbol{0} \in \mathbb{R}^{n}$,  $\mathcal{X}_z = \phi$ for all $z\in \mathcal{Z'}$, and $\mathcal{S} = \phi$ \
\STATE \textbf{Run:} 
\FOR{$x \in S^{-} $}
\FOR{$z$ \textup{s.t.} $z\subseteq x$ \textup{and} $z \in Z_{k,S}$}
\STATE $Z^- = Z^- \cup \{z\}$, \,\,  $Z_{k,S} = Z_{k,S}\setminus \{z\}$
\STATE $S_z = S_z \cup \{x\}$ 
\ENDFOR
\ENDFOR

\FOR{$x \in S^{+} $}
 %\If{$\exists$ $z\subseteq x$ such that either $z \in Z_{k,S}$ or $z\in Z^+$}{
  \FOR{$z \subseteq x$ \textup{such that}  $z \in Z_{k,S}$}
  \STATE $Z^+ = Z^+ \cup \{z\}$
  \ENDFOR
  %}
  %\Else{
  %Halt and Output \emph{empirical error not} $0$ and $\sign(v) \not\in F_k$.
  %\ganesh{I think it is a better idea to write $sign(v) \neq F_k$ as a part of result in theorem 4.8. } \nir{i'd even suggest - move the entire idea of identifying realizability outside of the algorithm. the algorithm works under the assumption of realizability. validating it is great - but external.}
  %}
\ENDFOR

\STATE Set $w(z)=
    \begin{cases}
        a_+ & \text{if } z \in Z^+ \hskip0.15\columnwidth  \darkgray{\triangleright \text{\textit{ implemented}}} \\
        a_- & \text{o.w. (implicitly)} \hskip0.09\columnwidth \darkgray{\text{\textit{as hash table}}}
    \end{cases}
$ 
\STATE \textbf{Return} $\hat{h}(z) = \sign( \sum\nolimits_{z' \in \sbstk_k(z)} w(z'))$.
% \Return $w$, and $h_w(z) = \sign(\sum_{z'\subseteq z, |z'| = k} w(z'))$ for all $z$ such that $k_1 \leq |z| \leq k_2$
\end{algorithmic}
\end{algorithm}
%\end{center}

% \paragraph{Correctness.
%\nir{consider giving informal definition of `induced realizability' or something, and then state the theorem using this as a condition.}
\begin{restatable}{theorem}{ermmin}\label{theorem: minimizing empirical error}
For any $k \in [k_2]$, if $Y$ is realizable then \textsc{Alg} returns an $\hat{h}$ that minimizes the empirical error.
%the minimum empirical error is zero then \textsc{Alg} returns an
%$h \in H_k$ that minimizes the empirical loss in Eq. \eqref{eq:emp_objective}.
%In particular, if $Y \in F_k$, that is, $Y$ is realizable then \textsc{Alg} returns $\hat{h}$ that minimizes the empirical error.
\end{restatable}
%\nir{what does it mean `if the min empirical error is 0'min over what?}
Proof in Appx.~\ref{secapp: strategic users who learn}.
%The main challenge here lies in showing that it suffices
%to work only with weights for $z$ of size exactly $k$,
%even though the output function $h$ must
%provide guarantees for all $z$ of size $|z| \in [k_1,k_2]$.
%For this, we use the concept of `lifting' operator that preserves guarantees (see Lem. \ref{lemma: h in H_k are lifts on size at most k}), 
%as well as the structural
%relation to induced functions (see Thm. \ref{theorem: nice structure}).
Note that our algorithm is exact:
it returns a true minimizer of the empirical 0/1 loss, assuming $Y$ is realizable.
% loss the empirical minimizer can be learnt via solving a linear program if the data is realizable and is NP-hard otherwise .
Additionally, \textsc{Alg} can be used to identify if there exists an $h \in H_k$ with zero empirical error;
at Step 15, for each $x \in S^+$ if there does not exist a $z\in Z_{k,S}$ or $z \in Z^+$ such that $z\subseteq x$
then from Lem. \ref{lem: 0 error charac} in Appx. \ref{secapp: learning proofs} there does not exist an $h$ with zero empirical error.
\begin{restatable}{lemma}{runtime} \label{lemma:runtime}
Let $n$ be the size of elements in $\X$, $m$ be the number of samples, and $k \leq k_2$ be the user's choice of complexity. Then 
\textsc{Alg} runs in  $O(m{n \choose k})$ time.
\end{restatable}
This runtime is made possible due to several key factors:
(i) only $k$-sized weights need to be learned,
(ii) all weights are binary-valued, and
(iii) loss queries are efficient in induced space.
Nonetheless, when $n$ and $k$ are large,
runtime may be significant,
and so $k$ must be chosen with care.
Fortunately, our results in Sec. \ref{sec:apx_err}
give encouraging evidence that learning with small $k$---even $k=1$, for which runtime is $O((mn)^2)$---is quite powerful (assuming $Y$ is realizable).

In the analysis, the $m{n \choose k}$ is made possible only 
since weights are sparse,
and since \textsc{Alg} operates on a finite sample set of size $m$.
Alternatively, if $m$ is large, then this can be replaced
with ${q \choose k}$.
This turns out to be necessary;
in Appx.~\ref{secapp: runtime lower bound} we show that, in the limit,
${q \choose k}$ is a lower bound.

\section{Balance of Power}\label{sec: balance of power}
\label{sec:balance}

Our final section explores the question:
what determines the balance of power between system and users?
We begin with the perspective of the user,
who has commitment power, but can only minimize the empirical error.
For her, the choice of complexity class $k$ is key
in balancing approximation error---how well (in principle) can functions $h \in H_k$ approximate $v$;
and estimation error---how close the empirical payoff of the learned $\hhat$ is to its expected value.
Our results give insight into how these types of error trade off as $k$ is varied (here we do not assume realizability).

For the system, the important factors are $k_1$ and $k_2$,
since these determine its flexibility in choosing representations.
Since more feasible representation mean more flexibility,
it would seem plausible that smaller $k_1$
and larger $k_2$ should help the system more.
However, our results indicate differently:
for system, \emph{smaller} $k_2$ is better,
and the choice of $k_1$ has limited effect on strategic users.
% \vineet{should we say it has limited effect on strategic users because 
% users learning subadditive functions could be affected?}
The result for $k_2$ goes through a connection to the user's choice of $k$;
surprisingly, smaller $k$ turns out to be, in some sense,
better for all.

% \todo{talk about `fairness` results?}

% \todo{if we end up not having subadditive, make sure we also take out places where we say things towards it}

% - for user, for fixed $k_1,k_2$, lower $k$ gives better apx error, but higher est error (and higher runtime)
% - for system, lower $k_2$ is better because it causes effective $k$ to be smaller, which is better for system. $k_1$ surprisingly doesn't matter.

% \outline{[DEF]: approximation error} 

%------------------------------------------------------------
\subsection{User's Perspective}
We begin by studying the effects of $k$ on user payoff.
Recall that users aim to minimize the expected error (Eq. \eqref{eq:exp_objective}):
\[
\errexp(h) = \expect{D}{\one{h(\rep_h(x)) \neq \sign(v(x))}},
\]
but instead minimize the empirical error (Eq. \eqref{eq:emp_objective}).
For reasoning about the expected error of the learned choice function $\hhat \in H_k$,
a common approach is to decompose it into two error types---\emph{approximation} and \emph{estimation}:
\begin{equation*}
\errexp(\hhat) = \underbrace{\errexp(h^*)}_{\text{approx.}} + %\quad
\underbrace{\errexp(\hhat) - \errexp(h^*)}_{\text{estimation}},\qquad
h^* = \argmin_{h' \in H_k} \errexp(h')
\end{equation*}
Approximation error describes the lowest error obtainable by functions in $H_k$;
this measures the `expressivity' of $H_k$, and is independent of $\hhat$.
For approximation error, we define a matching complexity structure
for value functions $v$,
and give several results relating the choice of $k$ and the complexity of $v$.
Estimation error describes how far the learned $\hhat$ is from the optimal
$h^* \in H_k$, and depends on the data size, $m$.
Here we give a generalization bound based on VC analysis.

% \begin{definition}
% Let $h^*_k = \argmax_{h \in H_k}\{\sum_{x \sim D} h(\phi_h(x)) \\= v(x)\}$. Then the approximation error of a strategic user learning over functions in $H_k$ is $1- U(h^*_k)$.
% \end{definition}

% \outline{[DEF]: estimation error}\\
% \begin{definition}
% The estimation error of a strategic user computing a $\tilde{h} \in H_k$ using Algorithm \toref\ is $U(h^*_k) - U(\tilde{h})$.
% \end{definition}

%------------------------------------------------------------
\subsubsection{User approximation error} \label{sec:apx_err}
To analyze the approximation error, we must be able to relate
choice functions $h$ (that operate on representations $z$) to the target
value function $v$ (which operates on items $x$).
To connect the two, we will again use induced functions,
for which we now define a matching complexity structure.
% Our algorithm works directly with induced functions as these
% provide direct access to loss computations,
% and our next set of result show that induced functions have useful structure.
% We start by relating the complexities of $h$ and $f_h$.
\begin{definition}
A function $f: \X \rightarrow \pmone$
has an \emph{induced complexity} of $\ell$
if exists a function $g:Z_\ell \rightarrow \pmone $ s.t.:
\[
f(x) = 
\begin{cases}
        1 & \text{\textnormal{if }} \,\, \exists  z \subseteq x, |z|=\ell \,\,\text{\textnormal{and}}\,\, g(z)=1 \\
        -1 & \text{\textnormal{o.w.}}
    \end{cases}
\]
and $\ell$ is minimal (i.e., there is no such $g':Z_{\ell-1} \rightarrow \pmone$).
% and there exists no $g':Z_{\ell-1} \rightarrow \pmone$ satisfying the above.
\end{definition}
We show in Lem.~\ref{lemma: induced complexity of fns in H_k is at most k} and Cor. \ref{cor: induced complexity k = F_k} that the induced complexity of a function $f$ captures the minimum $k\in [1,n]$
such that $f$ is an induced function of an $h\in H_k$.
% \nir{can we give intuition for what induced coplexity means, or tries to capture?
% can we connect this to results on modular/subadditive $v$ and relations to $H_k$?}
% \vineet{I think we can say the above blue line and remove the one following
% line.}
% Induced complexity applies to any binary function on $\X$;
% but for a induced functions, it has useful structure.
% % the complexities of $h$ and its induced $f_h$ are tied:
\begin{restatable}{lemma}{icoffn}\label{lemma: induced complexity of fns in H_k is at most k}
Let $k\le k_2$. Then for every $h \in H_k$,
the induced complexity of the corresponding $f_h$ is $\ell \leq k$.
\end{restatable}

\begin{restatable}{corollary}{fkequaltoick}\label{cor: induced complexity k = F_k}
Let $F_k = F_{H_k}$ be the induced function class of $H_k$,
as defined in Def. \ref{def:induced_class}.
Then:
\[
F_k = \{f : \X \rightarrow \pmone \,:\, f \textnormal{ has induced complexity } \le k\}.
\]
\end{restatable}
Proof of Cor. \ref{cor: induced complexity k = F_k} is in Appx.~\ref{secapp: bop}.
We now turn to considering the effect of $k$ on approximation error.
Since the `worthwhileness' function $Y(x) = \sign(v(x))$ operates on $x$,
we can consider its induced complexity,
which we denote by $\ell^*$
% the induced complexity of $y$, 
(i.e., $Y \in F_{\ell^*}$).
% (we will also call $\ell^*$ the `induced complexity of $v$').
The following result shows that if $\ell^* \leq k$,
then $H_k$ is expressive enough to perfectly recover $Y$.
\begin{restatable}{theorem}{zeroapproxerror}\label{theorem: 0 approx error}
If $\ell^* \leq k$ then the approximation error is $0$.
%The induced complexity of the induced function $f_{\tilde{h}}$ of $\tilde{h}$ is $\min\{\ell^*,k\}$, and in particular  
\end{restatable}

%Proof in Appx.~\ref{secapp: bop}.
%
One conclusion from Thm. \ref{theorem: 0 approx error}
is that if the user knows $\ell^*$,
then zero error is, in principle, obtainable;
another is that there is no reason to choose $k>\ell^*$.
In practice, knowing $\ell^*$ can aid the user in tuning $k$
according to computational (Sec. \ref{sec:algo})
and statistical considerations (Sec. \ref{sec:est_err}).
Further conclusions relate $\ell^*$ and $k_2$:
\begin{restatable}{corollary}{ellstarleqk} \label{corr:ellstar_leq_k2}
If $\ell^* \leq k_2$ and the distribution $D$ has full support on $\X$, then $k=\ell^*$ is the smallest $k$ that gives
zero approximation error.
\end{restatable}

\begin{restatable}{corollary}{ellstargtk} \label{corr:ellstar_g_k2}
If $\ell^*>k2$,
then the approximation error weakly increases with $k$, i.e.,
$\errexp(h^*_{k}) \leq \errexp(h^*_{k-1})$ for all $k \leq k_2$.
% \[
% \forall k \leq k_2, \qquad H_{k} \leq H_{k-1}  \,.
% \]
Furthermore, if the distribution $D$ has full support on $\X$ then no $k$ can achieve zero approximation error.
\end{restatable}
%Proofs in Appx.~\ref{secapp: bop}.
\begin{figure}[b!]
\centering
\includegraphics[width=0.92\columnwidth]{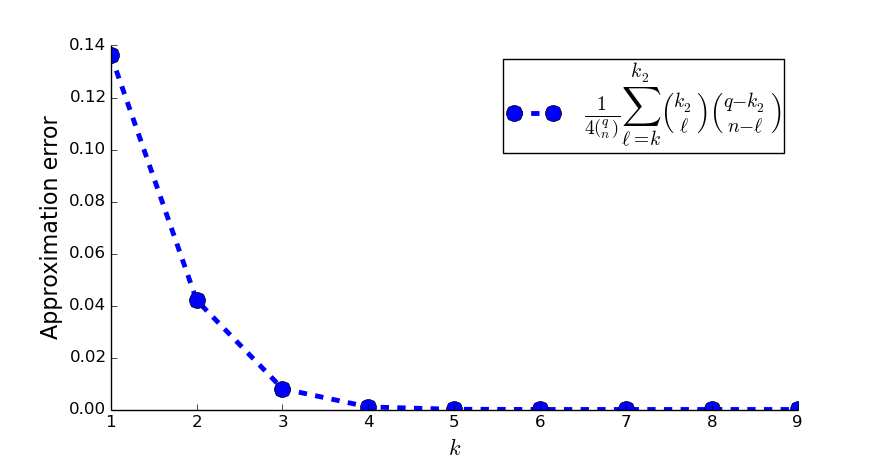}
\caption{Upper bound on approximation error showing diminishing returns. Parameters: $q=400, n=30$ and $k_2=10$.
%for increasing $k$ ($q=9, n=6$).
% \todo{ADD NICE-LOOKING PLOT}
}
\label{fig:diminishing_returns}
\end{figure}
%#############################################################################
Proofs in Appx. \ref{secapp: bop}. In general, Cor. \ref{corr:ellstar_g_k2} guarantees only weak improvement with
$k$. Next, we show that increasing $k$ can exhibit clear diminishing-returns behavior,
with most of the gain obtained at very low $k$. %The error function decreases for the $v$ constructed in Lem. \ref{theorem: diminishing returns} 
%as shown in Figure \ref{fig:diminishing_returns} for $q=400, n=30$, and $k_2=10$. The horizontal axis corresponds to $k$ (the complexity of the learning class), and the vertical access 
%corresponds to error of the optimal choice function from the hypothesis class. 
\begin{restatable}{lemma}{diminishingreturns}\label{theorem: diminishing returns}
Let $D$ be the uniform distribution over $\X$.
Then there is a value function $v$
% \blue{with induced complexity greater than $k_2$}
for which $\errexp(h^*_{k})$
% for every $k \in [1,k_2]$, 
% \[
% \forall k \leq k_2, \quad 
% \errexp(h^*_{k}) \leq  \frac{1}{4 {q \choose n}} \sum_{\ell=k}^{k_2} {k_2 \choose \ell} {q-k_2 \choose n-\ell} 
% \]
% where the RHS expression 
diminishes convexly with $k$.
\end{restatable}
The proof is constructive (see Appx. \ref{secapp: bop}). We construct a $v$ whose approximation error $h^*_k \in H_k$ is upper bounded by 
\[
\errexp(h^*_{k}) \leq  \frac{1}{4 {q \choose n}} \sum_{\ell=k}^{k_2} {k_2 \choose \ell} {q-k_2 \choose n-\ell} \,.
\] The diminishing returns property of upper bound is illustrated in Fig.  \ref{fig:diminishing_returns}.  Although Lem. \ref{theorem: diminishing returns} describes a special case,
we conjecture that this phenomena %of diminishing returns
applies more broadly. 
%In the Appendix, we give two additional results.
%The first result (Appx. \ref{secapp: diminishing returns}) shows that
%increasing $k$ can exhibit diminishing-returns behavior,
%with most gains obtained at very low $k$;
%in such cases, $v$ can be approximated well and at a low cost
%(computationally and statistically).
%The proof is constructive,
%and is illustrated in Figure \ref{fig:diminishing_returns} in Appx. \ref{secapp: diminishing returns}.

Our next result shows that
learning $k_1$-order functions can be as powerful as 
learning subadditive functions; %(Appx. \ref{secapp: subadditive functions});
hence, learning with $k=k_1$ is highly expressive.
Interestingly, the connection between general ($k$-order) functions
and subadditive functions is due to the
strategic response mapping, $\rep$.
\begin{restatable}{lemma}{subadditiveink}\label{thm:subadditive}
% Let $H_{\mathrm{SA}}$ be all threshold-subadditive functions,
Consider threshold-subadditive functions:
\[
H_{\mathrm{SA}} = 
\{\sign(g(z)) \,:\, g \textnormal{ is subadditive on subsets in } \Z\}
\]
Then for every threshold-subadditive $h_g \in H_{\mathrm{SA}}$, there is an
$h \in H_{k_1}$ for which $h(\phi_h(x)) = h_g(\phi_{h_g}(x))\,\, \forall x \in \X$.
\end{restatable}

\subsubsection{User estimation error} \label{sec:est_err}

For estimation error, we give generalization bounds based on VC analysis.
% \outline{generalization: VC of H is $\le q-choose-k$}
The challenge in analyzing functions in $H_k$
is that generalization applies to the \emph{strategic} 0/1 loss,
i.e., $\one{h(\rep_h(x)) \neq y}$, and so standard bounds
(which apply to the standard 0/1 loss) do not hold.
To get around this, 
our approach relies on directly analyzing the VC dimension
of the induced class, $F_k$ (a similar approach was taken
in \citet{sundaram2021pac} for SC).
This allows us to employ tools from VC theory,
which give the following bound.
\begin{restatable}{theorem}{generalizationbound}\label{theorem: generalization bound}
For any $k$ and $m$,
given a sample set $S$ of size $m$ sampled from $D$ and labeled by some $v$, we have 
\[
% U(h^*_k) - U(\tilde{h}) \leq %\epsilon,
\errexp(\hhat) - \errexp(h^*)  \leq %\epsilon,
\sqrt{ \frac{C({q \choose k} \log ({q \choose k}/\epsilon) + \log (1/\delta)}{m}}
\]
% where $\epsilon \leq \sqrt{ \frac{C({q \choose k} \log ({q \choose k}/\epsilon) + \log 1/\delta)}{m}}$ and
w.p. at least $1-\delta$ over $S$, and for a fixed constant $C$.
In particular,
\textsc{Alg} in Sec. \ref{sec:algo},
assuming $Y$ is realizable, returns an $\hhat \in H_k$
for which:
\[
% U(h^*_k) - U(\tilde{h}) \leq %\epsilon,
\errexp(\hhat)  \leq %\epsilon,
\sqrt{ \frac{C({q \choose k} \log ({q \choose k}/\epsilon) + \log (1/\delta)}{m}}
\]
% where $\epsilon \leq \sqrt{ \frac{C({q \choose k} \log ({q \choose k}/\epsilon) + \log 1/\delta)}{m}}$ and
w.p. at least $1-\delta$ over $S$, and for a fixed constant $C$.
% where $C$ is an independent .
\end{restatable}

The proof %(Appx.~\ref{secapp: bop})
relies on Thm. \ref{theorem: nice structure};
since $h$ and $f_h$ share weights,
the induced $F_k$ can be analyzed as a class of 
$q$-variate degree-$k$ multilinear polynomials.
Since induced functions already incorporate $\rep$,
VC analysis for the 0/1 loss can be applied.
Note that such polynomials have exactly ${q \choose k}$
degrees of freedom; hence the term in the bound.

% \nir{explain: \\
% - what this relies on \\
% - what VC tricks we used \\
% - how to think about the q-choose-k bits (relate to polynomials?)
% }

% \blue{
% - can think about $H_k$ as polynomials - but can't just use VC, because its not just 0/1 loss - the loss actually has $\rep$ in it \\
% - but we have a lemma!\\
%  -- for each $h$ there is an $f_h$ with same weights (Thm X) defined on $k$-sized $z$-s \\
% - this means that learning $f_h$ is also like polynomials - because it does look at 0/1, but accounts for $\rep$ inside $f$ \\
% - so we use VC of polynomials, and this is what we get \\
%  -- explains the q-choose-k thing... \\
%  -- f is a q-variate degree-k multilinear polynomial, the degrees of freedom is (q choose k).
% }

%------------------------------------------------------------
\subsection{System's Perspective}
The system's expressive power derives from its flexibility in choosing
representations $z$ for items $x$.
Since $k_1,k_2$ determine which representations are feasible,
they directly control the system's power to manipulate;
and while the system itself may not have direct control over $k_1,k_2$
(i.e., if they are set by exogenous factors like screen size),
their values certainly affect the system's ability to optimize engagement.
Our next result is therefore unintuitive:
for system, a smaller $k_2$ is better (in the worst case),
even though it reduces the set of feasible representations.
This result is obtained indirectly, by considering the effect of $k_2$ on
the user's choice of $k$.

% \nir{why go through induced functions here? i understand that this is how the proof would work, but can't we state this in terms of $k$ and $H_k$ directly?}
\begin{restatable}{lemma}{systemsmallk}\label{thm:small_k2_is_better}
%Assuming $D$ has full support over $\X$, for an $h^* \in H_k$ and $h' \in H_{k+1} \setminus H_k$,
% (by Lem. \toref, such $h'$ exists).
% Consider two user choice functions:
% $h$, whose corresponding $f_{h}$ has induced complexity $\ell$;
% and $h'$, whose corresponding $f_{h'}$ has larger induced complexity, $\ell' > \ell$.
%system's payoff against $h$ is higher than against $h'$.
%
%Suppose the value function of a strategic user $v$ is such that, for every $k\in [1,k_2]$, the optimal choice function in $H_k$ is not in $H_{k-1}$. Then the system's payoff against $h^*_k$ is higher than against $h^*_{k-1}$.
%$h^*_k \in H_k$ but $h^*_k \not\in H_{k-1}$. Then the system's payoff against $h^*_k$ is higher than against $h^*_{k-1}$.
%
There exists a distribution $D$ % over $\X$
and a value function $v$
% \blue{with induced complexity greater than $k_2$}
such that for all $k < k' \leq k_2$, system has higher payoff against the 
optimal $h^*_k \in H_{k}$ than against $h^*_{k'} \in H_{k'}$.
% for every $k \in [1,k_2]$, 
\end{restatable}
The proof is in Appx.~\ref{secapp: bop};
it uses the uniform distribution, and the value function from Thm. \ref{theorem: diminishing returns}. Recalling that the choice of $k$ controls the induced complexity $\ell$
(Cor. \ref{cor: induced complexity k = F_k}), and that users should choose $k$ to be no greater than $k_2$, % (\toref),
we can conclude the following (in a worst-case sense):
\begin{restatable}{corollary}{lowerkbetter}\label{corr:lower_k2_is_better}
For the system, lower $k_2$ is better.
\end{restatable} 

% \todo{for $k_1$, fix user wanting to learn subadditive functions. if $k_1$ decreases, then to still learn subadditvie function, the effective $k=k_1$ decreases. this means the effective functions that can be learned reduces, so it's harder to recover $v$. in this sense smaller $k_1$ is better for system,
% and worse for the user (in the worst case; as in other results) - both because the effective $k$ is smaller.}

Proof in Appx.~\ref{secapp: bop}. For $k_1$, it turns out that against strategic users---it is inconsequential.
This is since payoff to the strategic user is derived entirely from $k$, which
is upper-bounded by $k_2$, but can be set lower than $k_1$.
This invariance is derived immediately from how functions in $H_k$ are defined,
namely that $w(z)=a_-$ for all $z$ with $|z|<k$ (Def. \ref{def:k-order}).
% \nir{waiting for finalized results on subattivite functions}
This, however, holds when the strategic user chooses to learn
over $H_k$ for some $k$.
Consider, alternatively, a strategic user that decides to learn
subadditive functions instead.
In this case, Thm. \ref{thm:subadditive} shows that $k_1$
determines the users `effective' $k$; the smaller $k_1$,
the smaller the subset of subadditive functions that can be learned.
Hence, for user, smaller $k_1$ means worse approximation error.
% whereas for the system, smaller $k_1$ means larger payoff (in the worst case).
This becomes even more pronounced when facing a \naive\ user;
for her,
% For example, when facing a \naive\ user,
a lower $k_1$ means that system now has a large set of representations to choose from;
if even one of them has $v(z)=1$, the system can exploit this to increase its gains.
In this sense, as $k_1$ decreases, payoff to the system (weakly) improves.

\section{Discussion}
\label{sec:disc}

% \todo{need quick summary of paper? \\
% - user types, increasing complexity \\
% - algo + analysis \\
% - est, apx errs \\
% - bop
% }

Our analysis of the balance of power reveals a surprising conclusion:
for both parties, in some sense, simple choice functions are better.
% Put together, our results regarding the relation between $k$ and $k_2$ drive
% the conclusion that lower $k$ turns out to be, 
% in some sense, better for both parties.
For system, lower $k$ improves its payoff through how it relates to $k_2$
(Corollary \ref{corr:lower_k2_is_better}).
For users, lower $k$ is clearly
better in terms of runtime (Lemma~\ref{lemma:runtime})
and estimation error (Theorem \ref{theorem: generalization bound}),
and for approximation error, lower $k$ has certain benefits---as it relates to $\ell^*$ (Corollary \ref{corr:ellstar_leq_k2}),
and via diminishing returns (Theorem \ref{theorem: diminishing returns}).
% \todo{maybe subadditive, if true}
Thus, and despite their conflicting interests---to some degree, the incentives of the system and its users align.

But the story is more complex.
For users, there is no definitive notion of `better';
% as our results convey,
strategic users always face a trade-off,
and must choose $k$ to balance approximation, estimation, and runtime.
In principle, users are free to choose $k$ at will;
but since there is no use for $k>k_2$,
% they gain nothing by choosing $k$ to be larger than $k_2$,
a system controlling $k_2$ de facto controls $k$ as well.
This places a concrete restriction on the freedom of users to choose,
% and in a way which is inequitable:
and inequitably: 
for small $k_2$, users whose $v$ has complexity $\le k_2$
(i.e., having `simple tastes')
are less susceptible to manipulation than
users with $v$ of complexity $> k_2$
(e.g., fringe users with eclectic tastes)
(Theorem \ref{theorem: 0 approx error}, Corollaries. \ref{corr:ellstar_leq_k2} and \ref{corr:ellstar_g_k2}).
In this sense, the choice of $k_2$ also has implications on fairness.
We leave the further study of these aspects of strategic
representation for future work.

 From a purely utilitarian point of view,
 it is tempting to conclude that systems should always set $k_2$ to be low.
 But this misses the broader picture:
 although systems profit from engagement, 
 users engage only if they believe it is worthwhile to them,
 and dissatisfied users may choose to leave the system entirely
 (possibly into the hands of another).
 Thus, the system should not blindly act to maximize engagement;
 in reality, it, too, faces a tradeoff.
% From a purely utilitarian point of view,
% it is tempting to conclude that systems should always set $k_2$ to be low.
% But this misses the broader picture:
% although systems profit from engagement, 
% users engage only if they believe it is worthwhile to them,
% and dissatisfied users may choose to leave the system entirely
% (possibly into the hands of another).
% Thus, the system should not blindly act to maximize engagement;
% in reality, it, too, faces a tradeoff. %, and must choose $k_2$ appropriately.
% % This holds also for $k_1$, at least for the ---which we
% % is a fair description of the majority of users in reality.
% % Real system should not be fully strategic (and, perhaps, should even be slightly benevolent). 
% But reasoning about such considerations requires modelling
% an additional dimension: \emph{time}.
% Temporal aspects of strategic representation, and in particular
% in relation to user churn and fairness,
% are intriguing; we leave these for future work.

% \nir{@VINEET: please verify this is correct}
% We believe these may have connections to the recent line of work
% on performative prediction \tocite, as it relates to strategic calssification).

% \red{other: \\
% - popularity, converge to `shallow' content \\
% - differentiable algorithm - future
% }
\bibliographystyle{plainnat}
\bibliography{references.bib}

\newpage
\appendix

\section{{Additional Results}}\label{secapp: additional results}

\subsection{Agnostic User}\label{secapp: agnostic}

%\nir{@VINEET: in the main paper i put in an informal theorem, please don't forget to reference it here. i changed some notation so can you please verify that things remain consistent?
%also, if you could go over it and make sure i didn't break anything... that would be great}
%\ganesh{I have proved it in Appendix E. should we remove this from here?}

Theorem \ref{theorem: agnostic formal} stated below is the formal version of Theorem \ref{theorem: agnostic informal} in Section \ref{sec:Naive}. Theorem \ref{theorem: agnostic formal} shows that given a a large enough sample size $m$, an agnostic user's payoff would approach $\max\{\mu,1-\mu\}$, where $\mu = \expect{D}{y}$.% $  would make an informed choice by learning the item distribution. 
\begin{restatable}{theorem}{agnosticPAC}\label{theorem: agnostic formal}
Let $\frac{2}{2 + \sqrt{m}} \leq \delta < 1/8$  and $\tau =  \frac{\delta}{2(1-\delta)} + \sqrt{\frac{2\log(1/\delta)}{m}}$,  then agnostic user's  expected payoff guarantee is given by    $$ \begin{cases}
 \geq (1-\delta)(1-\mu) &   \text{if } \ \widehat{\mu} \leq 1/2 - \tau \\   \geq (1-\delta)\mu& \text{if } \  \widehat{\mu} \geq 1/2+\tau \\
  = 1/2& \text{Otherwise} 
\end{cases}
$$ 
\end{restatable}
%\agnosticPAC*
Before we prove the theorem, we state Hoeffding's inequality, which is a well known result from probability theory.
\begin{lemma}[Hoeffding]
\label{lem:hoeffdings}
Let $S_{m} = \sum_{i=1}^{m} X_{i}$ be the sum of $m$ i.i.d.~random variables with $X_i \in [0,1]$ and $\mu=\mathbb{E}[X_i]$ for all $i\in [m]$, then
\[
    \mathbb{P}(\frac{S_m}{m} - \mu \geq \varepsilon) \leq e^{ -2m\varepsilon^2} \,\,\, \text{ and} \,\,\,
    \mathbb{P}(\frac{S_m}{m} - \mu \leq -\varepsilon) \leq e^{ -2m\varepsilon^2}.
\]
\end{lemma}
We will use the following equivalent form of the above inequality. Let $\delta:=e^{-2m\varepsilon^2}$ i.e. $\varepsilon = \sqrt{\frac{2\log(1/\delta)}{m}}$ and $\widehat{\mu} = \frac{S_m}{m}$. Then we have with probability at-least $(1-\delta)$ we have
\begin{equation}
\label{eq:upperHoeffding}
 \mu \leq     \widehat{\mu} +  \sqrt{\frac{2\log(1/\delta)}{m}}~~~~\text{and}
\end{equation}
\begin{equation}
\label{eq:lowerHoeffding}
    \mu \geq   \widehat{\mu} -  \sqrt{\frac{2\log(1/\delta)}{m}}
\end{equation}

Now we are ready to give the proof of Theorem \ref{theorem: agnostic formal}
\begin{proof}[Proof of Theorem \ref{theorem: agnostic formal}]
We begin with the following supporting lemma.
\begin{lemma}
\label{lem:psibound}
Let $\frac{2}{2 + \sqrt{m}} \leq  \delta<1/8$, then 
$\tau < 1/2. $
\end{lemma}
\begin{proof}
The proof follows from following sequence of inequalities,
\begin{align*}
 \frac{2}{2 + \sqrt{m}} < \delta \iff m > 4(1/\delta - 1 )^2 \implies  m > 4(1/\delta - 1 )\log(1/\delta)  \iff \frac{\delta}{2(1-\delta)} >\frac{2\log(1/\delta)}{m} 
\end{align*}
Let $\gamma = \frac{\delta}{2(1-\delta)}$. We have $ \tau = \gamma + \sqrt{\gamma}$ which is an increasing function of $\delta$, so we  the maximum is achieved at $\delta=1/8$ and is given by  $1/\sqrt{14} + 1/14 < 1/2 $. This completes the proof of the lemma. 
\end{proof}
From Lemma \ref{lem:psibound} we have that $1/2 + \tau < 1$, hence there is a non-trivial range i.e. $\widehat{\mu}  \in [1/2 + \tau,1]$ where user assigns  $h(z)= +1$ for all $z$ with probability 1. Similarly, when $\widehat{\mu} \in [0, 1/2-\tau] $  user assigns $h(z)=-1 $ for all $z$ with probability 1. We will consider three cases separately. 

\textbf{Case 1 ($\widehat{\mu} \in [1/2+ \tau,1] $):}
From Hoeffding's inequality (Eq. \ref{eq:lowerHoeffding}) we have that with probability at-least $(1-\delta)$,
\begin{align*}
    \mu & \geq \widehat{\mu}  - \sqrt{\frac{2\log(1/\delta)}{m}} \\
    \implies \mu & \geq 1/2 + \frac{\delta}{2(1-\delta)} = \frac{1}{2(1-\delta)} 
\end{align*}
Hence, with probability at-least $(1-\delta)$ an agnostic user will get a payoff of $\mu$. Hence, the expected payoff in this case is at-least $(1-\delta)\mu \geq 1/2 \geq (1-\delta)(1-\mu)$. 

\textbf{Case 2 ($\widehat{\mu} \in [0, 1/2 - \tau] $):}
Similar to Case 1 here we use the tail bound given by Hoeffding's inequality (Eq. \ref{eq:upperHoeffding}) to get with probability at least $(1-\delta)$,
\begin{align*}
    \mu & \leq \widehat{\mu}  + \sqrt{\frac{2\log(1/\delta)}{m}} \\
    \implies \mu & \leq 1/2 - \frac{\delta}{2(1-\delta)} = \frac{1-2\delta}{2(1-\delta)}. 
\end{align*}
Hence, $(1-\mu)\geq \frac{1}{2(1-\delta)}$. The agnostic user guarantees a payoff of   $(1-\mu)$ with probability at least ($1-\delta$) in this case. Hence we have the payoff of $(1-\delta)(1-\mu)\geq 1/2 \geq (1-\delta)\mu$ in this case. 

\textbf{Case 3 ($\widehat{\mu} \in (1/2 - \tau,1/2 + \tau) $):}
Finally, in this case, the agnostic user chooses $h(z)=1$ for all $z\in\mathcal{Z}$ with probability $1/2$ and $h(z)=-1$ for all $z\in\mathcal{Z}$ with probability $1/2$. Hence, the expected payoff is given by 
$ \frac{1}{2}\mu + \frac{1}{2}(1-\mu) = 1/2$ irrespective of the true mean $\mu$ of positive samples.
\end{proof}

\subsection{Lifted Functions}\label{secapp: lifted functions}
The relation between choice functions and their induced counterparts
passes through an additional type of functions that operate
on sets of size \emph{exactly} $\ell$.
Denote $\Z_\ell = \{z \in 2^E : |z| =\ell\}$,
and note that all feasible representations can be partitioned as
$\Z = \Z_{k_1} \uplus \ldots \uplus \Z_{k_2}$.
We refer to functions that operate on single-sized sets 
as \emph{restricted functions}.
Our next result shows that choice functions in $H_k$
can be represented by restricted functions over $\Z_k$ that are `lifted' to operate on the entire $\Z$ space.
This will allow us to work only with sets of size exactly $k$.
\begin{lemma}\label{lemma: h in H_k are lifts on size at most k}
For each $h \in H_k$
there exists a corresponding $g: Z_{k} \rightarrow \pmone$
such that $h = \lift(g)$, where:
\begin{equation*}
\lift(g)(z) =
    \begin{cases}
        1 & \text{\textnormal{if }} k \le |z| \text{\textnormal{ and }} \\
        & \exists z' \subseteq z, |z'|=k \,\text{\textnormal{ s.t. }}\, g(z')=1 \\
        -1 & \text{\textnormal{o.w.}}
    \end{cases}
\end{equation*}
\end{lemma}
\begin{proof}
Let $h \in H_k$.  Then there is a weight function $w$ on sets of size at most $k$ such that either $w(z) \in (-1,0)$ or $w(z)  > \sum_{i \in [k]} {n \choose i}$, and 
\[
h(z) = \sign \left( \sum\nolimits_{z':z' \subseteq z, |z'| \le k} w(z') \right)
\]
Define $g:\Z_k \rightarrow \{-1,1\}$ such that for a $z\in \Z_k$, $g(z) = 1$ if $w(z) > 0$ and $g(z) = -1$ otherwise. It is easy to see from the choice of $w(z)$ that $h = \lift(g)$.
\end{proof}

\subsection{A Lower Bound on the Running Time}\label{secapp: runtime lower bound}

As stated in Lemma \ref{lemma:runtime}, the running time of our algorithm is $m({n \choose k})$.
We argued in Section \ref{sec: strat users and learning} that is made possible only 
since weights are sparse,
and since the algorithm operates on a finite sample set of size $m$.
If $m$ is large, then this expression can be replaced
with ${q \choose k}$.
We now show that, in the limit (or under full information),
the dependence on ${q \choose k}$ is necessary. The conclusion from Lemma \ref{lemma: basis for H_k} is that to find the loss minimizer,
any algorithm must traverse at least all such $h$;
since there exist ${q \choose k}$ such functions,
this is a lower bound.
This is unsurprising; $H_k$ is tightly related
to the class of multilinear polynomials,
whose degrees of freedom are exactly ${q \choose k}$.
\begin{restatable}{lemma}{basishk}\label{lemma: basis for H_k}
Consider a subclass of $H_k$ composed of choice functions $h$
which have $w(z)=a_+$ for exactly one $z$ with $|z|=k$,
and $w(z)=a_-$ otherwise.
% \blue{These form a basis}.
Then, for every such $h$, there exists a corresponding $v$,
such that $h$ is a unique minimizer (within this subclass)
of the error w.r.t.~$v$.
\end{restatable}
%
%\basishk*
\begin{proof}
Let $z_1$ and $z_2$ be distinct $k$ size subsets, and let $a_- \in (0,1)$ and $a_+ > \sum_{i\in [1,k]} {n \choose i}$. Further, let $\w_i$, $i\in [1,2]$ be a weight function that assigns $a_+$ to $z_i$ and $a_-$ to all other subsets of size at most $k$. Let $h_1$ and $h_2$ be two function in $H_k$ defined by the binary weighted functions $\w_1$ and $\w_2$ respectively. Observe that for $v_i = f_{h_i}$ the approximation error (see (Eq. \eqref{eq:exp_objective})) of $h_i$ is zero. Hence, to prove the lemma it is sufficient to show that $f_{h_1} \neq f_{h_2}$.

Suppose $f_{h_1} = f_{h_2}$. Since $z_1 \neq z_2$, there exists an $x \in \X$ such that $z_1 \subseteq x$ but $z_2 \subseteq x$. From Theorem \ref{theorem: nice structure} and the choice of $a_+$ and $a_-$, this implies $f_{h_1}(x)=1$ but $f_{h_2}(x) = -1$, and hence, a contradiction.
\end{proof}
%Proof in Appendix \toref.

%--------------------------------------------------
\section{Additional Related Work}\label{secapp: additional related work}

\paragraph{Learning set functions.}  Concept learning refers to learning a binary  function over hypercubes \citep{angluin1988queries} through a query access model. \citet{abboud1999learning} provide a lower bound on \emph{membership queries} to exactly learn a threshold function over sets where each element has small integer valued  weights.  Our learning framework admits large weights and  has only a sample access in contrast with the query access studied in this literature. \citet{feldman2009power}  show that the problem of learning set functions with sample access is  computationally hard.  However, we show (see Section 
\ref{sec: balance of power}) that the strategic setting is more nuanced; a more complex representations are disadvantageous for both user and the system. In other words, it is in the best interest of system to choose smaller (and much simpler) representations.  
A by-now classic work in learning theory studies the learnability from data of submodular (non-threshold) set functions~\citep{balcan2011learning}. Though we consider learning subadditive  functions in this work, an extension to submodular valuations is a natural extension. 
Learning set functions is in general hard, even for certain subcalsses
such as submodular functions.
\citet{rosenfeld2020predicting} show that it's possible to learn certain
parameterized subclasses of submodular functions, when
the goal is to use them for optimization.
But this refers to learning over approximate proxy losses;
whereas in our work, we show that learning is possible 
directly over the 0/1 loss.

\paragraph{Hierarchies of set functions.} 
\citet{conitzer2005combinatorial} (and independently,~\citet{ChevaleyreEEM08}) suggest a notion of $k$-wise dependent valuations, to which our Definition~\ref{def:hyper} is related. We also allow up to $k$-wise dependencies, but our valuations need not be positive and we focus on their sign (an indication whether an item is acceptable or not). Our set function valuations are also over item attributes rather than multiple items. Despite the differences, the definitions have a shared motivation:
\citet{conitzer2005combinatorial} believe that this type of valuation is likely to arise in many economic scenarios, especially since due to cognitive limitations, it might be difficult for a player to understand the inter-relationships between a large
group of items.
Hierarchies of valuations with limited dependencies/synergies have been further studied by \citet{abraham2012combinatorial,feige2015unifying} under the title `hypergraph valuations'. These works focus on monotone valuations that have only positive weights for every subset, and are thus mathematically different than ours.

\section{A Missing Proof from Section \ref{sec: setting}}\label{secapp: setting proof}
%\ganesh{I think we can get rid of this section as the observation is easy to see? Should we give a one line explanation somewhere in the main document and remove this section?} 
\samePayoff* 
\begin{proof}
The proof follows from the definition of best response (Eq. \ref{eq:br_rep}). Let $z_1, z_2 \in \rep_h(x)$. Then since $\rep_h$ consists of only best response, we have either $h(z_1) = h(z_2) = 1$, or $h(z_1) = h(z_2) = -1$. Hence, $h(z_1) = v(x)$ if and only if $h(z_2) = v(x)$ for any $z_1, z_2 \in \phi_h(x)$.
\end{proof}

\section{A Missing Proof from Section \ref{sec:Naive} and an Additional Example}
\label{secapp: learning proofs}
%\subsection{\Naive\ Users and a Benevolent System}
%\textbf{Benevolent System}:
%We call the system \emph{benevolent}  if for every choice function $h$ we have, $$ \phi_{h}(x) \in \argmax_{z \subseteq x, |z|=k} \mathbbm{1}\{ h(z)=v(x) \}.$$
\NaiveUser*
\begin{proof}
Since a \naive\ user plays $h(z) = \sign(v(z))$, for each $x\in \X$ the payoff of the user is maximized if in response the system plays a $z\subseteq x$ such that $\sign(v(z)) = \sign(v(x))$. 
Observe that, if there exists a $z\in \Z$ and $z\subseteq x$, such that $\sign(v(z)) = \sign(v(x))$ 
then $z \in \phi^{\mathrm{benev}}_h(x)$ and consequently the user's payoff is maximized for such an $x$.
Conversely, if there exists no $z\in \Z$ and $z\subseteq x$ such that $\sign(v(z)) = \sign(v(x))$,
then no truthful system can ensure more than zero utility for such an $x$.
Hence, a benevolent system maximizes the utility of a \naive\ user.
\end{proof}

We now present an additional example to show how a \naive\ user's choice function  can be manipulated by the strategic system and, as a consequence, the user may obtain arbitrarily small payoff against a strategic system.
%A \naive\ users choice function   can be manipulated by the strategic system and as a consequence, user may obtain arbitrarily small payoff in the worst case. Consider the following example.  

%In other words a naive User---often falsely--- believes that the System makes best decisions for her.

%Note that a weakly dominant strategy of a naive user is to be \emph{maximally}  truthful. That is, assign $h(z)$ as $v(x)$ to all possible $z$'s without conflicts. In case of conflicts,  $h(z)$ is assigned    based on   probability mass between $\{x \in X_z:v(x) = 1\}$ and $\{x \in X_z: v(x) = -1\}$. We will use  Naive users payoff as a baseline. 

%\outline{[RESULT] naive user can be tricked into bad choices (example)}
\begin{example}
Let $x_1 = \{a_1,a_2\}, x_2 = \{a_1,a_3\}, x_3 = \{a_1,a_4\}, x_4 = \{a_2,a_3\}, x_5 = \{a_3,a_4\}\}$ with $\sign(v(x_1)) = \sign(v(x_5))= \sign(v(a_2))= \sign(v(a_4)) = +1$ and $\sign(v(x_2))=\sign(v(x_3))= \sign(v(x_4)) = \sign(v(a_1))= \sign(v(a_3))=-1$. Further, let $k_1=k_2=1$ with $z_i = a_i$ as representations and  a distribution $D  = (\frac{\varepsilon}{4},\frac{\varepsilon}{4}, 1- \varepsilon, \frac{\varepsilon}{4}, \frac{\varepsilon}{4} )$ supported over $(x_1,x_2,x_3,x_4,x_5)$. \end{example}
A unique truthful  representation  for this instance is  $h = (-1,+1,-1,+1)$.  A strategic agent can manipulate a \naive\ agent into non-preferred choices by using a representation $ (a_2, a_1, a_4,a_2,a_4)$ for $(x_1,x_2,x_3,x_4,x_5)$.  Note here that a \naive\ agent expected  $z_1$ as a representation for $x_3$ since $h(z_1)=\sign(v(x_3))=-1$ and $h(z_4)=+1 \neq  \sign(v(x_3))$. However,    a strategic agent chose $a_4$ as under given $h$ we have $h(a_4)=1$. A \naive\ users payoff in this case is reduced to $\varepsilon$ which can be arbitrarily small. 

\section{Missing Proofs from Section \ref{sec: strat users and learning}}\label{secapp: strategic users who learn}
\orderkfunctions*
\begin{proof}
Define $k$ as follows: if $h(z) = -1$ for all $z \in \Z$ then $k = k_1$, and otherwise 
%\inbal{in the definition $k$ should be $k'$? also, need to explain why this is well defined I think}
$$
k = \max_{k'\in [k_1,k_2]}\{\exists z \text{ such that } |z| = k' \text{ and } h(z) = 1, \text{ but for all } z'\subset z \text{ and } z'\in \Z, h(z') = -1\}. 
$$
Define $h'$ as follows: For $|z|< k$, $h'(z) = -1$; for $|z| \ge k$ 
\begin{align*}
h'(z) = 1   & ~~~~~\text{if } \exists z': |z'| = k, z' \subseteq z \text{ and } h(z') = 1; \\
h'(z) = -1  & ~~~~~\text{otherwise}.
\end{align*}
First, we argue that $h'$ defined as above satisfies $h'(\phi_{h'}(x)) = h(\phi_h(x))$ for all $x\in \X$. Suppose $h(\phi_h(x)) = 1$. Then there exists $z \in \Z$ such that $z\subseteq x$  and $h(z) = 1$. From the choice of $k$, we may assume without loss of generality that $|z| = k$. Further, from the construction of $h'$, we have $h'(z) = 1$, 
%This implies there exists $z'$ such that $z \subseteq z' \subset x$ such that $h'(z') = 1$, 
and hence $h'(\phi_{h'}(x)) = h(\phi_h(x)) = 1$. Now suppose $h(\phi_h(x)) = -1$. Then for all $z \subseteq x$ we have $h(z) = -1$. In particular, for all $z \subseteq x$ such that $|z| = k$ we have $h(z) = -1$. This implies for all $z \subseteq x$ such that $|z| \geq k$ we have $h'(z) = -1$. This is because if there exists $z \subseteq x$ such that $|z| \geq k$ and $h'(z) = 1$ then from the definition of $h'$ there exists a $z' \subseteq z \subseteq x$ such that $|z'| = k$, and $h(z') = 1$ (a contradiction). Additionally, from definition, for all $z \subseteq x$ such that $|z| < k$ we have $h'(z) = -1$. Hence, if $h(\phi_h(x)) = -1$ then $h'(\phi_{h'}(x)) = -1$.

Now, we show that $h'$ is a $k$-order function. Let $a_- \in (-1,0)$ and $w(z) = a_{-}$ for all $z$ such that $|z| \leq k$ and $h'(z) = -1$. Further, for all $z$ such that $|z|=k$, if $h'(z) = 1$ then let $w(z) = a_{+} > \sum_{i \in [k]} {n \choose i}$. For all $z\in \Z$, if $|z| <k$ then by construction of $h'$, we have $h'(z) = -1$, and since for all $z'\in \sbstk_k(z)$, $w(z') = a_{-} < 0$ we have $\sum_{z'\in \sbstk_k(z)} w(z') < 0$. Hence, for all $z\in \Z$, if $|z| <k$
\[
h'(z) = \sign \left( \sum\nolimits_{z'\in \sbstk_k(z)} w(z') \right) = -1.
\]
Similarly, for all $z\in \Z$, if $|z| \geq k$ then by construction of $h'$, we have $h'(z) = 1$ if and only if there exists a $z' \subseteq z$, and $z' = |k|$ such that $h'(z) = h(z) = 1$. In particular, if $|z| \geq k$ and $h'(z) = 1$ then there exists a $z' \subseteq z$, and $|z'| = k$ such that $w(z') = a_+$. Since $a_{+} > \sum_{i \in [k]} {n \choose i}$, $a_- \in (-1,0)$, and $k_2 \leq n$, we have if $|z| \geq k$ and $h'(z) = 1$ then
\[
h'(z) = \sign \left( \sum\nolimits_{z'\in \sbstk_k(z)} w(z') \right) = 1.
\]
Finally, if $|z| \geq k$ and $h'(z) = -1$ then from the definition of $h'$ there does not exists a $z' \subseteq z$, and $|z'| = k$ such that $w(z') = a_+$. Since $a_- \in (-1,0)$, we have if $|z| \geq k$ and $h'(z) = -1$ then
\[
h'(z) = \sign \left( \sum\nolimits_{z'\in \sbstk_k(z)} w(z') \right) = -1.
\]
\end{proof}

\stricthierarchy*
\begin{proof}
Arbitrarily choose $u \subset E$ (recall $E$ is the ground set) such that $|u|=k$, and let $w(u) = a_{k,+} > \sum_{i \in [k]} {n \choose i}$.\footnote{Here we wish to distinguish between $a_+$ for $k$ and $k-1$ and hence we use $a_{k,+}$ instead of $a_+$.} 
Also for all $z \neq u$ and $|z| \leq k$, let $w(z) = a_- \in (-1,0)$. Let $h: \Z \rightarrow \pmone$ be defined as follows
\[
h(z) = \sign \left( \sum\nolimits_{z'\in \sbstk_k(z)} w(z') \right)
\]
From the definition of $H_k$, we have $h \in H_k$. We show that $h \not\in H_{k-1}$. First, observe that for all $z\in \Z$ $h(z) = 1$ if and only if $u \subseteq z$. Suppose $h \in H_{k-1}$. Then there is a weight function $w'$ on sets of size at most $k-1$ such that either $w'(z) = a_- \in (-1,0)$ or $w_z = a_{k-1,+} > \sum_{i \in [k-1]} {n \choose i}$, and 
\[
h(z) = \sign \left( \sum\nolimits_{z'\in \sbstk_{k-1}(z)} w'(z') \right)
\]
Let $z \in \Z$ be such that $u \subseteq z$. This implies $h(z) = 1$. Hence there exist a $u' \subseteq z$ such that $|u'| = k-1$ and $w'(u') = a_{k-1,+}$. Let $\tilde{z} \in \Z$ be such that $u' \subseteq \tilde{z}$ but $u \not\subseteq \tilde{z}$. Such a $\tilde{z}$ exists because $u \cap u' \neq u$. Further, as $u \not\subseteq \tilde{z}$, we have $h(\tilde{z}) = -1$. But since $u' \subseteq \tilde{z}$, we have from the choice of $a_{k-1,+}$ and $a_-$
\begin{align*}
    \sum\nolimits_{z' \in \sbstk_{k-1}(\tilde{z})} w'(z')  ~~ &> ~~ 0 \\
    \Rightarrow \sign \left( \sum\nolimits_{z' \in \sbstk_{k-1}(\tilde{z})} w'(z') \right) ~~& = ~~h(\tilde{z}) = ~~ 1.
\end{align*}
This gives a contradiction. Hence, $h\not\in H_{k-1}$.
\end{proof}

%%%%%%%%%%%%Comment%%%%%%%%%%%%%%%%%%
\begin{comment}
\hklifts*
\begin{proof}
Let $h \in H_k$ with weight function $\w$. Since $h\in H_k$, $\w$ defines weights over sets of size at most $k$,
such that either $w(z) \in (-1,0)$ or $w(z)  > \sum_{i \in [k]} {n \choose i}$, and 
\[
h(z) = \sign \left( \sum\nolimits_{z'\in \sbstk_k(z)} w(z') \right)\, .
\]
Define $g:Z_k \rightarrow \{-1,1\}$ such that for a $z\in Z_k$, $g(z) = 1$ if $w(z) > 0$ and $g(z) = -1$ otherwise. 
From the definition of $H_k$ and $\lift(g)$, for all $z \in \Z$ such that $|z| < k$ , $h(z) = \lift(g)(z) = -1$.
Further, for all $z \in \Z$ such that $|z| \geq k$, 
\begin{align*}
h(z) = 1 \, \iff&\, \exists z' \subseteq z \text{ such that }|z'| = k \text{ and } w(z') = a_+ > 0 \\
& \text{(from the choice of $a_+$)} \\
           \iff &\, \exists z' \subseteq z \text{ such that }|z'| = k \text{ and } g(z') = 1 \\
            &~~~~ \text{(from the definition of $g$)}\\
           \iff &\, \lift(g)(z) = 1 ~~~~\text{(from the definition of $\lift(g)$).}
\end{align*}
\end{proof}
\end{comment}
%%%%%%%%%%%%%%End Comment%%%%%%%%%%%%%%%%%%%%%%

\nicestructure*
\begin{proof}
Since $h \in H_k$, $\w$ satisfies the following two properties (see Definition \ref{def:k-order}):
\begin{enumerate}
\item  either $w(z) =a_- \in  (-1,0)$ or  $w(z) = a_+ > \sum_{i\in [k]}{n \choose i}$ ,
\item $w(z) = a_-$ for all $z$ having $|z|<k$.
\end{enumerate}
Further, from the definition of $f_h$, we have $f_h(x) = h(\phi_h(x))$. This implies 
\begin{equation*}
    f_h(x) = 1 \Longleftrightarrow \, \exists z \in \Z, \, z \subseteq x \text{ such that } h(z) = 1 \, .
\end{equation*}
From the the above two properties of the weights function, we have
\begin{equation*}
    h(z) = 1 \Longleftrightarrow \, \exists z' \subseteq z, \, |z'|=k \text{ such that } w(z') = a_+ > 0 \, .
\end{equation*}
From the above two equations we conclude that
\begin{equation*}
    f_h(x) = 1 \Longleftrightarrow \, \exists \, z \subseteq x, |z|=k \text{ such that } w(z) = a_+ > 0 \, .
\end{equation*}
Finally, the two properties of $\w$ ensure that
\[
f_h(x) = \sign \left( \sum\nolimits_{z \in \sbstk_k(x)} w(z) \right) \, .
\]
\end{proof}

\ermmin*
\begin{proof}
Throughout, for ease of notation, we use $x \in S$ to denote $x \in \{x_1, \ldots, x_m\}$.
Let $Z_{k} = \{ z \,:\, |z|=k, \exists x\in S ~  z\subseteq x\}$. 
Recall $Z_{k,S}$ is equal to $Z_k$ at the beginning of the algorithm. 
Also, for each $z\in Z_k$, let $\X_z = \{x \in S \mid z\subseteq x\}$.
The following lemma characterizes the training set for which there exists an $h\in H_k$ with zero empirical error.
\begin{lemma}\label{lem: 0 error charac}
There exists an $h\in H_k$ with zero empirical error if and only if for all $x\in S^+$ there exists a $z\in Z_k$ and $z\subseteq x$ such that $z \not\subseteq x'$ for all $x' \in S^-$.
\end{lemma}
\begin{proof}
Suppose there exists an $h\in H_k$ with zero empirical error. 
Since $h\in H_k$, from Lemma \ref{lemma: h in H_k are lifts on size at most k} we have there exists a $g:Z_k \rightarrow \pmone$ such that $h = \lift(g)$.
We state the following observation, and its proof follows from the definition of $\lift(g)$. 
\begin{observation}\label{obs: lift def}
\begin{enumerate}
    \item For every $z\in \Z$ such that $h(z) = 1$ there is a $z' \in Z_k$ such that $z' \subseteq z$ and $g(z') = 1$,
    \item For every $z\in \Z$ such that $h(z) = -1$, it must be that for every $z'\in Z_k$, $z' \subseteq z$ we have $g(z') = -1$.
\end{enumerate}
\end{observation}
Further, as the empirical error for $h$ is zero, we have the following observation.
\begin{observation}\label{obs: empirical error 0}
\begin{enumerate}
    \item For every $x \in S^{+}$ there is a $z\in \Z$ and $z\subseteq x$ such that $h(z) = 1$, 
    \item For every $x \in S^{-}$, it must be that for every $z\in \Z$, $z\subseteq x$ we have $h(z) = -1$. 
\end{enumerate}
\end{observation}
\begin{proof}
For every $x\in S^+$, since the empirical error is zero, we have $h(\phi_h(x)) = 1$. From the definition of $\phi_h$, this implies there is a $z\in \Z$ and $z\subseteq x$ such that $h(z) = 1$. Similarly,
for every $x\in S^-$, since the empirical error is zero, we have $h(\phi_h(x)) = -1$. Again from the definition of $\phi_h$, it must be that for every $z\in \Z$, $z\subseteq x$ we have $h(z) = -1$. 
\end{proof}
Hence, from Observations \ref{obs: lift def} and \ref{obs: empirical error 0}, we have for every $x \in S^{+}$ there is a $z\in Z_k$, $z\subseteq x$ such that $g(z) = 1$.
Similarly, for every $x \in S^{-}$ it must be that for every $z\in Z_k$, $z\subseteq x$ we have $g(z) = -1$.
Hence, for every $x\in S^+$ there exists a $z\in Z_k$ and $z\subseteq x$ such that $z \not\subseteq x'$ for all $x' \in S^-$.

Conversely, suppose for every $x\in S^+$ there exists a $z\in Z_k$ and $z\subseteq x$ such that $z \not\subseteq x'$ for all $x' \in S^-$.
Then define $g: Z_k \rightarrow \pmone$ as follows: a) for all $z\in Z_k$ such that $z\subseteq x$ for an $x\in S^{-}$, let $g(z) = -1$,
b) for all $z \in Z_k$, such that $z\subseteq x$ for an $x\in S^{+}$ and $z\not\subseteq x'$ for an $x'\in S^{-}$, let $g(z) = 1$, 
c) for all $z \in Z_k$, such that $z\not\subseteq x$ for any $x \in S$, let $g(z) = -1$.
From the supposition, we have that for every $x\in S^+$, there is a $z\in Z_k$ and $z\subseteq x$ such that $g(z) = 1$.
Now define $h = \lift(g)$. To show that the empirical error of $h$ is zero, it is sufficient to show that for every $x\in S^-$\,
$h(\phi_h(x)) = -1$, and for every $x\in S^+$\, $h(\phi_h(x)) = 1$.
Let $x \in S^-$. From the definition of $g$, for every $z \in Z_k$
such that $z\subseteq x$, $g(z) = -1$. Hence, from the definition of $\lift$, we have for every $z\in \Z$ such that $z\subseteq x$, $h(z) =-1$. Now from
the definition of best response, we have $h(\phi_h(x)) = -1$. Similarly, if $x \in S^+$ from our supposition and the definition of $g$, we have there exists 
a $z\in Z_k$ such that $z\subseteq x$ and $g(z)=1$. Hence, from the definition of $\lift$, there exists 
a $z\in \Z$ such that $z\subseteq x$ and $h(z)=1$. Finally, from the definition of best response , we have $h(\phi_h(x)) = 1$.
\end{proof}
Now, if $Y \in F_k$ then there exists an $h\in H_k$ such that the induced function of $h$ is equal to $Y$, that is, $f_h = Y$. This implies there exists an $h\in H_k$ which attains zero empirical error on 
the training set. 
Since empirical error is always non-negative, such an $h$ minimizes the empirical error in this case. 
Hence, from Lemma \ref{lem: 0 error charac}, it follows that if $Y$ is realizable then  for all $x\in S^+$ at Step 17 of \textsc{Alg} there is either a $z\in Z^+$ and $z\subseteq x$, or $z\in Z_{k,S}$ and $z \subseteq x$.
%Hence, Steps 18 and 19 are only executed if there is no $h$ with zero empirical error.
%Hence, to prove the first part of the lemma it is sufficient to show that if the algorithm never executes Steps 18 and 19 then it returns an $h$ with zero empirical error. 
Now, observe that at the beginning of Step 22, set $Z^+$ satisfies the following: 
\begin{equation}\label{eqn: erm thm 1}
    z\in Z^+ \,\,\, \iff \,\,\, \exists \, x\in S^+ \text{ such that } z\subseteq x  \text{ and } \not\exists x' \in S^- \text{ such that } z\subseteq x'.
\end{equation}
Further, at Step 22, for a $z\in Z_k$, $w(z) = a_+$ if $z\in Z^+$. This implies 
\begin{equation}\label{eqn: erm thm 2}
    w(z) = a_+ \,\,\, \iff \,\,\, \exists \, x\in S^+ \text{ such that } z\subseteq x  \text{ and } \not\exists x' \in S^- \text{ such that } z\subseteq x'.
\end{equation}
Also, from Theorem \ref{theorem: nice structure}, the induced function $f_{\hat{h}}$ corresponding to the returned $\hat{h}$ is given as
\begin{equation}\label{eqn: erm thm 3}
f_{\hat{h}}(x) = \sign \left( \sum\nolimits_{z \in \sbstk_k(x)} w(z) \right)\,.
\end{equation}
To complete the proof of theorem, we show that $f_{\hat{h}}(x_i) = y_i$ for every $x_i\in S$. Suppose $x\in S^-$.
Then from Equations and \ref{eqn: erm thm 1} and \ref{eqn: erm thm 2}, for every $z\subseteq x$ and $|z| \leq k$ we have $w(z) = a_- < 0$, and hence from Equation \ref{eqn: erm thm 3} for $f_{\hat{h}}$
we have $f_{\hat{h}}(x) = y = -1$. Similarly, suppose $x \in S^+$. Then from Equation \ref{eqn: erm thm 2}, there exists $z \subseteq x$, $|z| = k$ such that $w(z) = a_+$.
Hence, from Equation \ref{eqn: erm thm 3}, and noting that $a_+ > \sum_{i \in [k]} {n \choose i}$ and $a_- \in (-1,0)$ we have $f_{\hat{h}}(x) = y =1$.
% Hence, from the above argument the algorithm returns a $\hat{h}$ with minimum empirical error.
\end{proof}

\runtime*
\begin{proof}
In the first two for loops, for each $x \in S^+$ (or in $S^-$) the internal for loop runs for $O({n \choose k})$ time.  
Since $|S| \leq m$, this is a total of at most $O(m{n \choose k})$ operations.
Similarly, Step 22 places weights on at most $m{n \choose k}$ subsets, and hence runs in $O(m{n \choose k})$ time.
Hence, \textsc{Alg} runs in $O(m{n \choose k})$ time.
\end{proof}

\section{Missing Proofs from Section \ref{sec: balance of power}} \label{secapp: bop}

\icoffn*
\begin{proof}
Let $\ell = \min_{k' \in [1,k]}\{\text{there exists a } g:Z_{\ell} \rightarrow \pmone \text{ such that } h = \text{lift}(g)\}$. From Lemma \ref{lemma: h in H_k are lifts on size at most k}, we know $\ell \leq k$. Further, assume $g: Z_{\ell} \rightarrow \pmone$ is such that $h = \text{lift}(g)$. Now, from the definition of $f_h$, we have for all $x\in \X$, $f_h(x) = 1$ if and only if there exists a $z\in \Z$ such that $z\subseteq x$ and $h(z) = 1$. 
%Since $h \in H_k$, for all $x\in \X$ $f_h(x) = 1$ if and only if there exists a $z\in \Z$, $|z| \geq k \geq \ell$ such that $z\subset x$ and $h(z) = 1$. 
Since $h = \text{lift}(g)$, $f_h(x) = 1$ if and only if there exists a $z \in Z_{\ell}$ such that $z \subseteq x$ and $g(z) = 1$. 
This implies the induced complexity of $f_h$ is $\ell \leq k$.
\end{proof}

\fkequaltoick*
\begin{proof}
From Lemma \ref{lemma: induced complexity of fns in H_k is at most k}, we know that functions in $F_k$ have induced complexity at most $k$. We show that if $f$ has induced complexity at most $k$ then there is an $h \in H_k$ such that $f = f_h$. Let the induced complexity of $f$ be equal to $\ell \leq k$. Then there exists a $g: Z_\ell \rightarrow \pmone$ such that 
\begin{align}\label{eqn: cor. 4.9 1}
f(x) = 1  \iff \exists z \in Z_\ell  
\text{ such that } z\subseteq x \text{ and } g(z) = 1.
\end{align}
Let $h = \lift(g)$. First we show that $f(x) = f_h(x)$ for all $x \in \X$. Since $h$ is a lift of $g$, if $g(z') = 1$ for a $z' \in Z_\ell$ then for all $z \in \Z$
such that $z'\subseteq z$, we have $h(z) = 1$. 
\begin{align}\label{eqn: cor. 4.9 2}
h(z) = 1  \iff \exists z' \in Z_\ell  
\text{ such that } z'\subseteq z \text{ and } g(z') = 1. 
\end{align}
Hence, from Equations \ref{eqn: cor. 4.9 1} and \ref{eqn: cor. 4.9 2} for all $x \in \X$, $f(x) = 1$ if and only if there exists $z \in \Z$ such that $z\subseteq x$ and $h(z) =1$.
From the definition of induced function,
this implies $f(x) = f_h(x)$ for all $x \in \X$.

To show $h\in H_k$, we construct a weight function $\w$ on sets of size at most $k$. For $z \in 2^E$ and $|z|<k$, let $w(z) = a_- \in (-1,0)$. 
For $z\in Z_k$, let
\[
w(z)  
\begin{cases}
     = a_+ >   \sum_{i \in [1,k]} {n \choose i} & \text{\textnormal{if }} \,\, \exists z' \subseteq z, |z'|=\ell \,\,\text{\textnormal{and}}\,\, g(z)=1 \\
     = a_-  & \text{\textnormal{o.w.}}
    \end{cases}
\]
Now from Equation \ref{eqn: cor. 4.9 2}, $h(z) = 1$ if and only if there exists $z' \in Z_\ell$ such that $z'\subseteq z$ and $g(z') = 1$. Hence, from the definition of $\w$, $h(z) = 1$ if and only if there exists $z' \in Z_\ell$ such that $z'\subseteq z$ and $w(z') = a_+$. In particular, since $a_+ > \sum_{i=1}^k {n \choose i}$ and $a_- \in (0,1)$, we have
\[
h(z) = \sign \left( \sum\nolimits_{z':z' \subseteq z, |z'| \le k} w(z') \right)\, .
\]
\end{proof}

\zeroapproxerror*
\begin{proof}
Since the induced complexity of $v$ is $\ell^*$, there is a function $g:Z_{\ell^*} \rightarrow \pmone $ s.t.:
\[
v(x) = 
\begin{cases}
        1 & \text{\textnormal{if }} \,\, \exists  z \subseteq x, |z|=\ell^* \,\,\text{\textnormal{and}}\,\, g(z)=1 \\
        -1 & \text{\textnormal{o.w.}}
    \end{cases}
\]
Let $a_+ > \sum_{i \in [1,k]} {n \choose i}$ and $a_- \in (0,1)$,
and define the weight function $\w$ on sets of size at most $k$ as follows: a) if $|z| < k$ then let $w(z) = a_-$,
b) if $|z| = k$ and there exists a $z' \subseteq z$ such that $g(z') = 1$ then $w(z) = a_+$, and 
c) if $|z| = k$ and there does not exist a $z' \subseteq z$ such that $g(z') = 1$ then $w(z) = a_-$.
Now define $h$ using $\w$ as follows:
\[
h(z) = \sign \left( \sum\nolimits_{z' \in \sbstk_k(z)} w(z') \right)\, .
\]
We now show that for each $x \in \X$, $h(\phi_h(x)) = f_h(x) = v(x)$ implying $h^*_k = h$. 
Suppose $f_h(x) = 1$ for an $x \in \X$. 
Then there exists a $z \in \Z$ such that $z\subseteq x$ and $h(z) = 1$. 
From Theorem \ref{theorem: nice structure}, and the choice of $a_+$ and $a_-$ we have that there exists a 
$z \subseteq x$, $|z| =k$ such that $w(z) = a_+$.
From the construction of $w$ this implies
there exists a $z \subseteq x$, $|z| =\ell^*$ such that $g(z) = a_+$.
But from the above definition of $v$ this implies $v(x) = 1$.
Similarly, we can argue, if $f_h(x) = -1$ then $v(x) = -1$ for any $x \in \X$.
Hence, $h(\phi_h(x)) = v(x)$ for each $x\in \X$ implying $h^*_k = h$ and 0 approximation error for $h$.
\end{proof}

\ellstarleqk*
\begin{proof}
In the proof of Theorem \ref{theorem: 0 approx error}, we show that for $k =\ell^*$, we have zero approximation error. Hence, to prove the corollary it is sufficient to show that for a $k<\ell^*$
the approximation error is not zero. Suppose there is an $h \in H_k$ such that $\varepsilon(h) =0$ and $k<\ell^*$.
Since the distribution $D$ has full support, this implies
$f_h(x) = v(x)$ for all $x\in \X$.
Hence, the induced complexity of $v$ is at most $k < \ell^*$ giving a contradiction.
\end{proof}

\ellstargtk*
\begin{proof}
The approximation error weakly decreases because $H_{k-1} \subseteq H_k$ for all $k \leq k_2$.
Also, from the proof of Corollary \ref{corr:ellstar_leq_k2}, it is clear that no $k$ can achieve zero approximation error. 
\end{proof}

\diminishingreturns*
\begin{proof}
We construct a $v$ such that the approximation error for $h^*_k \in H_k$ is as given below
\[
\errexp(h^*_{k}) =  \frac{1}{4 {q \choose n}} \sum_{\ell=k}^{k_2} {k_2 \choose \ell} {q-k_2 \choose n-\ell} \,.
\]
It is easy to see that $\errexp(h^*_{k})$ diminishes convexly with $k$ (see Fig. \ref{fig:diminishing_returns}).
We choose $k_2$ elements $e_1, e_2, \ldots , e_{k_2} \in E$ (the ground set), and let $z_e$ be the $k_2$ size subset consisting of these $k_2$ elements. For a $v: \X \rightarrow \mathbb{R}$, let $\X_{v}^+ = \{x \in \X | \sign(v(x))=1\}$ and $\X_{v}^- = \{x \in \X | \sign(v(x))=-1\}$. We first show that there exists a $v$ with the following two properties:
\begin{enumerate}
    \item if $x \in  \X_{v}^+$ then there exists a $z \subseteq z_e$ such that $z \subseteq x$.
    \item For $k\in [1,k_2]$, let $\X_{k} = \{x\in \X | \exists z \subseteq z_e, |z| =k, \text{ and } z\subseteq x\}$. 
    Then $\X_{v}^+ \cap \X_k = \frac{3}{4}(\sum_{\ell = k}^{k_2} {k_2 \choose \ell}{q-k_2 \choose n-\ell})$, for every $k \in [1,k_2]$.
    \item For every $z \subseteq z_e$, let $\X_z =\{x\in \X \mid z \subseteq x\}$. Then $|\X_v^+ \cap \X_{z}| = \frac{3}{4}{q-k \choose n-k}$, where $|z|=k$.
\end{enumerate}
We construct such a $v$ iteratively. We begin by making the following observation.
\begin{observation}
For each $k \in [1,k_2]$, $|\X_{k}| = \sum_{\ell = k}^{k_2} {k_2 \choose \ell}{q-k_2 \choose n-\ell}$.
\end{observation}
\begin{proof}
Recall $\X$ consists of size $n$ subsets of $E$. For a $k\in [1,k_2]$ we wish to choose $n$ size subsets of $E$ that contain a $z\subseteq z_e$, $|z| = k$. This equivalent to choosing a fixed $\ell\geq k$ size subset of $z_e$ and then choosing the remaining $n-\ell$ elements from the $q-k_2$ elements (not part of $z_e$) in $E$. For every $\ell \geq k$ we can choose $\ell$ size subset of $z_e$ in ${k_2 \choose \ell}$ ways, and for each such choice we can choose the remaining $n-\ell$ elements in ${q-k_2 \choose n-\ell}$ ways. Since, this holds for any $\ell \in [k,k_2]$, we have $|\X_{k}| = \sum_{\ell = k}^{k_2} {k_2 \choose \ell}{q-k_2 \choose n-\ell}$.
\end{proof}
\textbf{Constructing} v: The idea is to iteratively add elements in $\X$ to $\X_{v}^+$, that is, iteratively determine the $x \in \X$ such that $\sign(v(x)) = 1$. In the first round, we arbitrarily choose $\frac{3}{4}{q-k_2 \choose n-k_2}$ from $\X_{k_2}$ and add it to $\X_{v}^+$, and the remaining $\frac{1}{4}{q-k_2 \choose n-k_2}$ are added to $\X_{v}^-$. At round $k$, assume we have constructed a $v$ satisfying the above three properties for $k' > k$, that is, 
\begin{enumerate}
    \item if $x \in  \X_{v}^+$ then there exists a $z \subseteq z_e$ such that $z \subseteq x$.
    \item  $\X_{v}^+ \cap \X_{k'} = \frac{3}{4}(\sum_{\ell = k'}^{k_2} {k_2 \choose \ell}{q-k_2 \choose n-\ell})$, for every $k' \in [k+1,k_2]$.
    \item For every $z \subseteq z_e$, let $\X_z =\{x\in \X \mid z \subseteq x\}$. Then $|\X_v^+ \cap \X_{z}| = \frac{3}{4}{q-k' \choose n-k'}$, where $|z|=k'>k$.
\end{enumerate}

Hence, at round $k$, we have ${k_2 \choose k}{q-k_2 \choose n-k}$ elements in $\X_k$ which are not yet in $\X_{v}^+$ or $\X_{v}^-$. From these elements in $\X_k$, for every $k$ size subset $z \subseteq z_e$ we arbitrarily choose $\frac{3}{4}{q-k_2 \choose n-k}$ elements containing $z$ and add the remaining $\frac{1}{4}{q-k_2 \choose n-k}$ elements to $\X_{v}^-$. Now, observe that $v$ satisfies the first two properties for every $k'\in [k,k_2]$ after this procedure. 
We argue $v$ satisfies the third property for any $z \subseteq z_e$, such that $|z| = k$.
The $n$ size sets in $\X$ containing a $z \subseteq z_e$, such that $|z| = k$, can 
be partitioned into sets containing different $\ell >=k$ size subsets of $z_e$.
In particular, we have the following combinatorial equality
\[
{q-k' \choose n-k'} = \sum_{\ell=k'}^{k_2} {k_2-k' \choose \ell-k'}{q-k_2 \choose n-\ell}
\]
In the above expression, ${q-k_2 \choose n-\ell}$ corresponds to the number of $n$ size sets that contain only a
a specific $\ell\geq k'$ size subset of $z_e$. Since our iterative procedure ensures from each such partition at least
$\frac{3}{4}$ fraction of $x$ is added to $\X_v^+$, we have that $v$ satisfies the third property.

\textbf{Optimal} $h^*\in H_k$: From the construction of $v$, it is clear that the optimal $h^* \in H_k$ for the above constructed $v$, for any $k \in [1,k_2]$ satisfies the following: for every $z\in \Z$, $h^*(z) = 1$ if and only if there exists a $z' \subseteq z_e$, $|z'|=k$, and $z' \subseteq z$. Further as $D$ is the uniform distribution, for such an $h^*$:
\[
\errexp(h^*_{k}) =  \frac{1}{4 {q \choose n}} \sum_{\ell=k}^{k_2} {k_2 \choose \ell} {q-k_2 \choose n-\ell} \,.
\]

\end{proof}

\subadditiveink*
\begin{proof}

Let $h\in H_{\mathrm{SA}}$ with a corresponding $g: \Z \rightarrow \mathbb{R}$
such that $h(z) = \sign(g(z))$ for all $z\in \Z$.
Choose an $a_+ > \sum_{i=1}^{k_1} {n \choose i}$, and $a_- \in (0,1)$.
Define a weight function $\w$ on sets of size at most $k_1$ as follows:
\[
w(z) = \begin{cases}
       a_+ & \text{\textnormal{if }} \, |z|=k_1, \, h(z)=1  \\
        a_- & \text{\textnormal{o.w.}}  
\end{cases}
\]
Let $h'\in H_{k_1}$ be the function defined by the binary weight $\w$ as defined above.
We argue that for every $x\in \X$, if $h(\phi_h(x)) = h'(\phi_{h'}(x))$.
For every $x\in \X$, $h(\phi_h(x)) = 1$ if and only if there is a $z\in \Z$
and $z\subseteq x$ such that $h(z) = 1$. Since $g$ is sub-additive, we have
\begin{equation}
    0 \leq g(z) \leq \sum_{z' \subseteq z, z'\neq z, z'\in \Z} g(z')\, .
\end{equation}
A simple recursive argument implies $h(\phi_h(x)) = 1$ if and only if there 
is a $z\subseteq x$ such that $|z|=k_1$ and $h(z) = 1$, and hence $w(z) = a_+$.
Hence, from Theorem \ref{theorem: nice structure} this implies, $h(\phi_h(x)) = 1$ if and only if $h'(\phi_{h'}(x)) = 1$.
\end{proof}

\generalizationbound*
\begin{proof}
We first argue that the VC dimension of $H_k$ is at most ${q \choose k}$. 
Let $d = \sum_{i \in [1,k]} {n \choose i}$, index the vectors in $\{0,1\}^d$ by $z \subseteq E$ (the ground set), such that $|z|\leq k$. 
Then each $z\in \Z$ can be represented by a binary vector $e_z \in \{0,1\}^d$, with the entry indexed by a $z'$ being $1$ if and only if $z' \subseteq z$.
Further, let $w \in \{a_-,a_+\}^d$ be a binary weighted vector with $a_-$ and $a_+$ as in Def. \ref{def:k-order}.
Then from the definition of $H_k$, for each $h \in H_k$, there is a $w_h \in \{a_-,a_+\}^d$ such that 
a) $h(z) = \sign(\langle w, e_z\rangle)$ for all $z\in Z$, and
b) the entry of $w$ indexed by a $z'$ with $|z'| < k$ is $a_-$. From this we observe that the VC dimension of $H_k$ is at most ${q \choose k}$, 
since each $h\in H_k$ is decided by the realization of binary weights on entries indexed by the ${q \choose k}$ sets.
Now the first part of the theorem follows by noting that the first bound is the agnostic PAC generalization guarantee for an algorithm minimizing the empirical
error in the standard classification setting with VC dimension at most ${q \choose k}$.
To prove the second part, we have $Y\in F_k$, and hence the approximation error is zero, that is, $\errexp(h^*) = 0$ (from Lemma \ref{lem: 0 error charac}). Further, \textsc{Alg} minimizes the empirical error (Theorem \ref{theorem: minimizing empirical error}), and returns an $\hat{h}$ with zero empirical error. 
\end{proof}

\systemsmallk*
\begin{proof}
The $v$ is constructed as in the proof Lemma \ref{theorem: diminishing returns}. We recall notations from the proof of Lemma \ref{theorem: diminishing returns}: $z_e$ is a $k_2$ size subset. Further, in the proof of Lemma \ref{theorem: diminishing returns} we argued that for $k\in [1,k_2]$, $h^*_k$ is such that for all $z\in \Z$, $h^*(z) = 1$ if and only if there exists a $k$ size $z' \subseteq z$ which is also a subset of $z_e$. 

Now let $k,k' \in [1,k_2]$ such that $k<k'$. Since $D$ is the uniform distribution,
to show system's utility is more for $k$ compared to $k'$
it is sufficient to show that 
\[
\sum_{x \in \X}{\one{h^*_k(\phi_{h^*_k}(x))=1}} > \sum_{x\in \X}{\one{h^*_{k'}(\phi_{h^*_{k'}}(x))=1}}
\]
From the proof of Lemma \ref{theorem: diminishing returns} and Theorem \ref{theorem: nice structure}, it follows that
\[
\sum_{x \in \X}{\one{h^*_k(\phi_{h^*_k}(x))=1}} = \sum_{x \in \X}{\one{f_{h^*_k}(x)=1}} = \sum_{\ell = k}^{k_2} {k_2 \choose \ell}{q-k_2 \choose n-\ell}
\]
Similarly,
\[
\sum_{x\in \X}{\one{h^*_{k'}(\phi_{h^*_{k'}}(x))=1}} = \sum_{\ell = k'}^{k_2} {k_2 \choose \ell}{q-k_2 \choose n-\ell}
\]
Since $k < k'$, we have
\[
\sum_{x \in \X}{\one{f_{h^*_k}(x)=1}} = \sum_{\ell = k}^{k_2} {k_2 \choose \ell}{q-k_2 \choose n-\ell} > \sum_{\ell = k'}^{k_2} {k_2 \choose \ell}{q-k_2 \choose n-\ell}
\]
implying system's utility is more for $k$ compared to $k'$.
\end{proof}

\lowerkbetter*
\begin{proof}
In Lemma \ref{thm:small_k2_is_better}, we showed there exists a user with $v$ such that for all $k,k' \in [1,k_2]$ and $k<k'$, the system has better utility against the 
optimal choice function in $H_{k}$ than in $H_{k'}$.
Since the choice of $k$ the user can make is bounded by $k_2$, a lower $k_2$
maximizes the worst-case payoff to the system.
\end{proof}
\end{document}